\documentclass[twoside,11pt]{article}

\usepackage[abbrvbib]{jmlr2e}

\usepackage[shortlabels]{enumitem}

\usepackage{amsfonts}
\usepackage{mathrsfs}
\usepackage{amsmath}
\usepackage{color}

\definecolor{jgGreen}{rgb}{0.0, 0.5, 0.0}

\definecolor{pink}{rgb}{0.5, 0.0, 0.25}

\newcommand{\alteta}[1]{}

\newcommand{\ignore}[1]{}

\newcommand{\cB}{{\mathcal B}}

\newcommand{\cD}{\mathcal{D}}

\newcommand{\cL}{\mathcal{L}}
\newcommand{\cM}{\mathcal{M}}

\newcommand{\sU}{{\mathscr U}}

\newcommand{\ba}{{\mathbf a}}
\newcommand{\bb}{{\mathbf b}}

\newcommand{\bh}{{\mathbf h}}

\newcommand{\bx}{{\mathbf x}}
\newcommand{\by}{{\mathbf y}}

\DeclareMathOperator*{\esssup}{ess\,sup}

\DeclareMathOperator{\proj}{Proj}

\DeclareMathOperator{\interior}{int}
\DeclareMathOperator{\supp}{supp}

\DeclareMathOperator{\dist}{dist}

\def\Rset{\mathbb{R}}
\def\Nset{\mathbb{N}}

\newcommand{\PP}{{\mathbb P}}
\newcommand{\QQ}{{\mathbb Q}}

\newcommand{\one}{{\mathbf{1}}}
\newcommand{\zero}{{\mathbf{0}}}
\newcommand{\Wball}[1]{{\cB^\infty_{#1}}}
\newcommand{\tWball}[1]{{\tilde{\cB}^\infty_{#1}}}

\newcommand{\mres}{\mathbin{\vrule height 1.6ex depth 0pt width
0.13ex\vrule height 0.13ex depth 0pt width 1.3ex}}

\newcommand{\ov}{\overline}

\newcommand{\td}{\tilde}
\newcommand{\e}{\epsilon}

\usepackage{thm-restate}

\newtheorem{assumption}{Assumption}

\newcommand{\prm}{R_\phi^\e}
\newcommand{\dl}{\bar R_\phi}

\jmlrheading{24}{2023}{1-41}{4/23}{12/23}{23-0456}{Natalie Frank \& Jonathan Niles-Weed}

\ShortHeadings{Existence and Minimax Theorems for Adversarial Losses}{Frank \& Niles-Weed}
\firstpageno{1}

\begin{document}
\title{Existence and Minimax Theorems for Adversarial Surrogate Risks in Binary Classification}

\author{\name Natalie S. Frank \email nf1066@nyu.edu \\
       \addr Courant Institute\\
       New York University\\
       New York, NY 10012, USA
       \AND \\
\name Jonathan Niles-Weed \email jnw@cims.nyu.edu \\
       \addr Courant Institute and Center for Data Science\\
       New York University\\
       New York, NY 10012, USA
       }

\editor{Gabor Lugosi}

\maketitle

\begin{abstract}
	We prove existence, minimax, and complementary slackness theorems for adversarial surrogate risks in binary classification.
	These results extend recent work of \citet{PydiJog2021}, who established analogous minimax and existence theorems for the adversarial classification risk.
	We show that their conclusions continue to hold for a very general class of surrogate losses; moreover, we remove some of the technical restrictions present in prior work.
	Our results provide an explanation for the phenomenon of transfer attacks and inform new directions in algorithm development.
\end{abstract}


\section{Introduction}\label{sec:intro}

    Neural networks are state-of-the-art methods for a variety of machine learning tasks including image classification and speech recognition. However, a concerning problem with these models is their susceptibility to \emph{adversarial attacks}: small perturbations to inputs can cause incorrect classification by the network \citep{szegedy2013intriguing,biggio2013evasion}. This issue has security implications; for instance, \citet{GuDolanGavittGarg17} show that a yellow sticker can cause a neural net to misclassify a stop sign. 
    Furthermore, one can find adversarial examples that generalize to other neural nets; these sort of attacks are called \emph{transfer attacks}. In other words, an adversarial example generated for one neural net will sometimes be an adversarial example for a different neural net trained for the same classification problem \citep{FlorianPapernotGoodfellowetal17,DemontisMelisetal18,kurakin2017adversarial,RozsaGuntherBoult16,PapernotMcDanielGoodfellowetal16}. This phenomenon shows that access to a specific neural net is not necessary for generating adversarial examples. One method for defending against such adversarial perturbations is \emph{adversarial training}, in which a neural net is trained on adversarially perturbed data points \citep{kurakin2017adversarial,madry2019deep,wang2019}. However, adversarial training is not well understood from a theoretical perspective.

    From a theoretical standpoint, the most fundamental question is whether it is possible to design models which are robust to such attacks, and what the properties of such robust models might be.
    In contrast to the classical, non-adversarial setting, much is still unknown about the basic properties of optimal robust models.
    For instance, in the context of binary classification, several prior works study properties of the \emph{adversarial classification risk}---the expected number of classification errors under adversarial perturbations.
   A minimizer of the adversarial risk is optimally robust against adversarial perturbations of the data; however, it is not clear under what conditions such a minimizer exists.
   Recently, 
    \citet{AwasthiFrankMohri2021}, \citet{BungertGarciaMurray2021}, and \citet{PydiJog2021} all showed existence of a minimizer to the adversarial classification risk under suitable assumptions, and characterized some of its properties.
    A crucial observation, emphasized by \citet{PydiJog2021}, is that minimizing the adversarial classification risk is equivalent to a \emph{dual} robust classification problem involving uncertainty sets with respect to the $\infty$-Wasserstein metric.
    This observation gives rise to a game-theoretic interpretation of robustness, which takes the form of a zero-sum game between a classifier and an adversary who is allowed to perturb the data by a certain amount. As noted by \citet{PydiJog2021}, this interpretation has implications for algorithm design by suggesting that robust classifiers can be constructed by jointly optimizing over classification rules and adversarial perturbations.
    
    This recent progress on adversarial binary classification lays the groundwork for a theoretical understanding of adversarial robustness, but it is limited insofar as it focuses only on minimizers of the adversarial classification risk.
    In practice, minimizing the empirical adversarial classification risk is computationally intractable; as a consequence; the adversarial training procedure typically minimizes an objective based on a \emph{surrogate} risk, which is chosen to be easier to optimize.
    In the non-adversarial setting, the properties of surrogate risks are well known~\cite[see, e.g.][]{BartlettJordanMcAuliffe2006}, but in the adversarial scenario, existing results for the adversarial classification risk fail to carry over to surrogate risks.
    In particular, the existence and minimax results described above are not known to hold.
    We close this gap by developing an analogous theory for adversarial surrogate risks. 
    Our main theorems (Theorems~\ref{th:complimentary_slackness}--\ref{th:existence_primal}) establish that strong duality holds for the adversarial surrogate risk minimization problem, that solutions to the primal and dual problems exist, and that these optimizers satisfy a complementary slackness condition.

    These results suggest explanations for empirical observations, such as the existence of transfer attacks.            Specifically, our analysis suggests that adversarial examples are a property of the data distribution rather than a specific model. 
    In fact, our complimentary slackness theorem states that certain attacks are the strongest possible adversary against \emph{any} function that minimizes the adversarial surrogate risk, which might explain why adversarial examples tend to transfer between trained neural nets. Furthermore, our theorems suggest that a training algorithm should optimize over neural nets and adversarial perturbations simultaneously.  Adversarial training, the current state of the art method for finding adversarially robust networks, does not follow this procedure. The adversarial training algorithm tracks an estimate of the optimal function $\tilde f $ and updates $\tilde f$ through gradient descent. To update $\tilde f$, the algorithm first finds \emph{optimal} adversarial examples at the current estimate $\tilde f$. See the papers \citep{kurakin2017adversarial,madry2019deep,GoodfellowShlensSzegedy14} for more details on adversarial training. Finding these  adversarial examples is a computationally intensive procedure. On the other hand, algorithms for optimizing minimax problems in the finite dimensional setting alternate between primal and dual steps \citep{mokhtari2019unified}. This observation suggests that designing an algorithm that optimizes over model parameters and adversarial perturbations simultaneously is a promising research direction. \citet{GarciaTrillosGarciaTrillos23,WangChizat2023exponentially,domingoenrichbruna2021meanfield} adopt this approach, and one can view the minimax results of this paper as a mathematical justification for the use of surrogate losses in such algorithms. 
    
    Lastly, our theorems are an important first step in understating statistical properties of surrogate losses. Recall that one minimizes a surrogate risk because minimizing the original risk is computationally intractable. If a sequence of functions which minimizes the surrogate risk also minimizes the classification risk, then that surrogate is referred to as a \emph{consistent risk}. Similarly, if a sequence of functions which minimizes the  adversarial surrogate risk also minimizes the adversarial classification risk, then that surrogate is referred to as an \emph{adversarially consistent risk}. Much prior work studies the consistency of surrogate risks \citep{BartlettJordanMcAuliffe2006,Lin2004,Steinwart2007, LongServedioH-consistency,ZhangAgarwal,zhang04}. 
    Alarmingly, \cite{MeunierEttedguietal22} show that a family of surrogates used in applications is not adversarially consistent. In follow up work, we show that our results can be used to characterize adversarially consistent supremum-based risks for binary classification \citep{FrankNilesWeed22}, strengthening results on calibration in the adversarial setting \cite{bao2021calibrated, MeunierEttedguietal22,AwasthiFrankMao2021}.  
    
\section{Related Works}\label{sec:related_works}
    This paper extends prior work on the adversarial Bayes classifier. \citet{PydiJog2021} first proved multiple minimax theorems for the adversarial classification risk using optimal transport and Choquet capacities, showing an intimate connection  between adversarial learning and optimal transport. Later, follow-up work used optimal transport minimax reformulations of the adversarial learning problem to derive new algorithms for adversarial learning. \citet{TrillosJacobsKim22} reformulate adversarial learning in terms of a multi-marginal optimal transport problem and then apply existing techniques from optimal transport to find a new algorithm. \citet{GarciaTrillosGarciaTrillos23,WangChizat2023exponentially,domingoenrichbruna2021meanfield} propose ascent-descent algorithms based on optimal transport and use mean-field dynamics to analyze convergence. These approaches leverage the minimax view of the adversarial training problem to optimize over model parameters and optimal attacks simultaneously. \citet{gao2022wasserstein} use an optimal transport reformulation to find regularizers that encourage robustness. \citet{WongSchmidtKolter2019,wu2020stronger} use Wasserstein metrics to formulate adversarial attacks on neural networks.
    
    Further work analyzes properties of the adversarial Bayes classifier. \citet{AwasthiFrankMohri2021}, \citet{BhagojiNitinCullina2019lower}, and \citet{BungertGarciaMurray2021} all prove the existence of the adversarial Bayes classifier, using different techniques. \citet{YangRashtchian2020} studied the adversarial Bayes classifier in the context of non-parametric methods.  \citet{PydiJog2019} and \citet{BhagojiNitinCullina2019lower} further introduced methods from optimal transport to study adversarial learning. 
    Lastly, \citep{trillosMurray2020} 
    give necessary and sufficient conditions describing the boundary of the adversarial Bayes classifier.     Simultaneous work \citep{LiTelgarsky2023achieving} also proves the existence of minimizers to adversarial surrogate risks using prior results on the adversarial Bayes classifier.
    
    The adversarial training algorithm is also well studied from an empirical perspective. \citet{DemontisMelisetal18} discussed an explanation of transfer attacks on neural nets trained using standard methods, but did not extend their analysis to the adversarial training setting. \citep{wang2019,kurakin2017adversarial,madry2019deep} study the adversarial training algorithm and improving the runtime. Two particularly popular attacks used in adversarial training are the FGSM attack \citep{GoodfellowShlensSzegedy14} and the PGD attack \citep{madry2019deep}. More recent popular variants of this algorithm include \citep{ShafahiNajibietal19,XieWuVanDerMaaten18,KannanKurakinGoodfellow18,WongRiceKolter2020}. 

    \section{Background and Notation}\label{sec:background}
    \subsection{Adversarial Classification}\label{sec:adv_classification} 
        This paper studies binary classification on $\Rset^d$ with two classes encoded as $-1$ and $+1$. Data is distributed according to a distribution $\cD$ on $\Rset^d\times\{-1,+1\}$. We denote the marginals according to the class label as $\PP_0(S)=\cD(S\times \{-1\})$ and $\PP_1(S)=\cD(S\times\{+1\})$. Throughout the paper, we assume  $\PP_0(\Rset^d)$ and $\PP_1(\Rset^d)$ are finite but not necessarily  that $\PP_0(\Rset^d)+\PP_1(\Rset^d)=1$.

        To classify points in $\Rset^d$, algorithms typically learn a real-valued function $f$ and then classify points $\bx$ according to the sign of $f$ (arbitrarily assigning the sign of 0 to be $-1$).
 The \emph{classification error}, also known as the \emph{classification risk}, of a function $f$ is  \begin{equation}\label{eq:standard_zero_one_loss}
    R(f)= \int\one_{f(\bx)\leq 0}d\PP_1+\int  \one_{f(\bx)> 0} d\PP_0.    
        \end{equation}
        Notice that finding minimizers to $R$ is straightforward: define the measure $\PP=\PP_0+\PP_1$ and let $\eta=d\PP_1/d\PP$. Then the risk $R$ can be re-written as 
        \[R(f) =\int \eta(\bx) \one_{f(\bx)\leq 0}+(1-\eta(\bx))\one_{f(\bx)>0}d\PP.\]
        Hence a minimizer of $R$ must minimize the function $C(\eta(\bx),\alpha)= \eta(\bx) \one_{\alpha\leq 0}+(1-\eta(\bx))\one_{\alpha>0}$ at each $\bx$ $\PP$-a.e.
        The optimal Bayes risk is then 
        \[\inf_f R(f)=\int C^*(\eta)d\PP\]
        where $C^*(\eta) =\inf_\alpha C(\eta,\alpha)=\min(\eta,1-\eta)$.

        This paper analyzes the \emph{evasion attack}, in which an adversary knows both the function $f$ and the true label of the data point, and can perturb each input before it is evaluated by the function $f$. To constrain the adversary, we assume that perturbations are bounded by $\e$ in a norm $\|\cdot\|$. Thus a point $\bx$ with label $+1$ is misclassified if there is a perturbation $\bh$ with $\|\bh\|\leq \e$ for which $f(\bx+\bh)\leq 0$ and a point $\bx$ with label $-1$ is misclassified if there is a perturbation $\bh$ with $\|\bh\|\leq \e$ for which $f(\bx+\bh)>0$. Therefore, the \emph{adversarial classification risk} is 
        \begin{equation}\label{eq:adv_zero_one_loss}
     R^\e(f)=\int\sup_{\|\bh\|\leq \e}\one_{f(\bx+\bh)\leq 0}d\PP_1+\int  \sup_{\|\bh\|\leq \e}\one_{f(\bx+\bh)> 0} d\PP_0
        \end{equation}
        which is the expected proportion of errors subject to the adversarial evasion attack. The expression $\sup_{\|\bh\|\leq \e } \one_{f(\bx+\bh)\leq 0}$ evaluates to 1 at a point $\bx$ iff $\bx$ can be moved into the set $[f\leq 0]$ by a perturbation of size at most $\e$. Equivalently, this set is the Minkowski sum $\oplus$ of $[f\leq 0]$ and $\ov{B_\e(\zero)}$. For any set $A$, let $A^\e$ denote
\begin{equation}\label{eq:e_operation_def}
    A^\e=\{\bx\colon \exists \bh\text{ with } \|\bh\|\leq \e \text{ and }\bx+\bh\in A\}=A\oplus \ov{B_\e(\zero)}=\bigcup_{\ba\in A} \ov{B_\e(\ba)}.
\end{equation}
    This operation `thickens' the boundary of a set by $\e$. With this notation, \eqref{eq:adv_zero_one_loss} can be expressed as $R^\e(f)=\int \one_{\{f\leq 0\}^\e} d\PP_1+\int \one_{\{f>0\}^\e}d\PP_0$.
        
    Unlike the classification risk $R$, finding minimizers to $R^\epsilon$ is nontrivial.
        One can re-write $R^\e$ in terms of $\PP$ and $\eta$ but the resulting integrand cannot be minimized in a pointwise fashion.
        Nevertheless, it can be shown that minimizers of $R^\epsilon$ exist
        \citep{AwasthiFrankMohri2021,BungertGarciaMurray2021,PydiJog2021, FrankNilesWeed22}.

    \subsection{Surrogate Risks}\label{sec:surrogate_risks}
    As minimizing the empirical version of risk in \eqref{eq:standard_zero_one_loss} is computationally intractable, typical machine learning algorithms minimize a proxy to the classification risk called a \emph{surrogate risk}. In fact, \citet{BenDavidEironLong2003} show that minimizing the empirical classification risk is NP-hard in general. One popular surrogate is
    \begin{equation}\label{eq:standard_phi_loss}
        R_\phi(f)=\int \phi(f)d\PP_1+\int\phi(-f) d\PP_0
    \end{equation}
    where $\phi$ is a decreasing function.\footnote{Notice that due to the asymmetry of the sign function at 0 in \eqref{eq:standard_zero_one_loss}, $R_\phi$ is not quite a generalization of $R$.}
    To define a classifier, one then threshholds $f$ at zero.
    There are many reasonable choices for $\phi$---one typically chooses an upper bound on the zero-one loss which is easy to optimize. We make the following assumption on $\phi$:
    \begin{assumption}\label{as:phi}
        The loss $\phi$ is non-increasing, non-negative, lower semi-continuous, and $\lim_{\alpha \to \infty} \phi(\alpha)=0$. 
    \end{assumption}
  A particularly important example, which plays a large role in our proofs, is the exponential loss $\psi(\alpha)=e^{-\alpha}$, which will be denoted by $\psi$ in the remainder of this paper.
    Assumption~\ref{as:phi} includes many but not all all surrogate risks used in practice. Notably, some multiclass surrogate risks with two classes are of a somewhat different form, see for instance \citep{TewariBartlett2007} for more details.

    In order to find minimizers of $R_\phi$, we rewrite the risk in terms of $\PP$ and $\eta$ as 
    \begin{equation}\label{eq:standard_phi_risk_eta}
        R_\phi(f)= \int \eta(\bx) \phi(f(\bx)) +(1-\eta(\bx)) \phi(-f(\bx)) d\PP
    \end{equation}
    Hence the minimizer of $R_\phi$ must minimize $C_\phi(\eta,\cdot)$ pointwise $\PP$-a.e., where 
    \begin{equation*}
        C_\phi(\eta,\alpha)=\eta\phi(\alpha)+(1-\eta)\phi(-\alpha).
    \end{equation*}

    In other words, if one defines $C_\phi^*(\eta)=\inf_\alpha C_\phi(\eta,\alpha)$, then a function $f^*$ is optimal if and only if
    \begin{equation}\label{eq:standard_optimal_f}
        \eta(\bx)\phi(f^*(\bx))+(1-\eta(\bx))\phi(-f^*(\bx))=C_\phi^*(\eta(\bx)) \quad \PP\text{-a.e.}
    \end{equation}
    Thus one can write the minimum value of $R_\phi$ as 
    \begin{equation}\label{eq:R_phi_min}
        \inf_f R_\phi(f)=\int C_\phi^*(\eta) d\PP.    
    \end{equation}
To guarantee the existence of a function satisfying \eqref{eq:standard_optimal_f}, we must allow our functions to take values in the extended real numbers $\ov \Rset=\Rset\cup \{-\infty,+\infty\}$.
Allowing the value $\alpha = + \infty$ is necessary, for instance, for the exponential loss $\psi(\alpha) = e^{-\alpha}$:
when $\eta =1$, the minimum of $C_\psi(1,\alpha)=e^{-\alpha}$ is achieved at $\alpha = +\infty$. In fact, one can express a minimizer as a function of the conditional probability $\eta(\bx)$ using \eqref{eq:standard_optimal_f}. For a loss $\phi$, define $\alpha_\phi: [0,1]\to \ov R$ by 
\begin{equation}\label{eq:alpha_phi_definition}
    \alpha_\phi(\eta)=\inf \{\alpha: \alpha\text{ is a minimizer of } C_\phi(\eta,\cdot)\}.
\end{equation}
    Lemma~\ref{lemma:C_phi_minimizers} in Appendix~\ref{app:C_phi_minimizers}
shows that the function $\alpha_\phi$ is montonic and $\alpha_\phi(\eta)$ is in fact a minimizer of $C_\phi(\eta,\cdot)$. Thus the function 
\begin{equation}\label{eq:minimizer_from_eta}
    f^*(\bx)=\alpha_\phi(\eta(\bx))    
\end{equation}
 is measurable and satisfies \eqref{eq:standard_optimal_f}. Therefore, $f^*$ must be a minimizer of the risk $R_\phi$.
    
    Similarly, rather directly minimizing the adversarial classification risk, practical algorithms minimize an \emph{adversarial surrogate}. The adversarial counterpart to \eqref{eq:standard_phi_loss} is
    \begin{align}
        &R_\phi^\e(f)
        =\int \sup_{\|\bh\|\leq \e} \phi(f(\bx+\bh))d\PP_1 +\int \sup_{\|\bh\|\leq \e} \phi(-f(\bx+\bh))d\PP_0.\label{eq:adv_phi_loss}
    \end{align}
    
    Due to the definitions of the adversarial risks \eqref{eq:adv_zero_one_loss} and \eqref{eq:adv_phi_loss}, the operation of finding the supremum of a function over $\e$-balls is central to our subsequent analysis. For a function $g$, we define 
    \begin{equation}\label{eq:S_e_def}
        S_\e(g)(\bx)=\sup_{\|\bh\|\leq \e} g(\bx+\bh)
    \end{equation}

    Using this notation, one can re-write the risk $R_\phi^\e$ as 
     \begin{align*}
        &R_\phi^\e(f)
        =\int S_\e(\phi\circ f)d\PP_1 +\int S_\e(\phi\circ -f)d\PP_0
    \end{align*}
    By analogy to~\eqref{eq:standard_phi_risk_eta}, we equivalently write the risk $R_\phi^\e$ in terms of $\PP$ and $\eta$:
    \begin{equation}\label{eq:not_pw_adv_surrogate}
        R_\phi^\e(f)=\int \eta(\bx) S_\e(\phi\circ f)(\bx)+(1-\eta(\bx))S_\e(\phi\circ -f)(\bx) d\PP.
    \end{equation}
    However, unlike~\eqref{eq:standard_phi_risk_eta}, because the integrand of $R_\phi^\e$ cannot be minimized in a pointwise manner, proving the existence of minimizers to $R_\phi^\e$ is non-trivial.     In fact, unlike the adversarial classification risk $R^\e$, there is little theoretical understanding of properties of minimizing $R_\phi^\e$. 

    \subsection{Measurability}
    In order to define the adversarial risks $R^\e$ and $R_\phi^\e$, one must show that $S_\e(\one_A), S_\e(\phi \circ f)$ are measurable. To illustrate this concern, \citet{PydiJog2021} show that for every $\e>0$ and $d>1$, there is a Borel set $C$ for which the function $S_\e(\one_C)(\bx)$ is not Borel measurable. However, if $g$ is Borel, then $S_\e(g)$ is always measurable with respect to a larger $\sigma$-algebra called the \emph{universal $\sigma$-algebra} $\sU(\Rset^d)$. Such a function is called \emph{universally measurable}. We prove the following theorem and formally define the universal $\sigma$-algebra in Appendix~\ref{app:meas}.  
        \begin{restatable}{theorem}{thmeas}\label{th:meas}
    If $f$ is universally measurable, then $S_\e(f)$ is also universally measurable.
\end{restatable}
In fact, in Appendix~\ref{app:meas}, we show that a function defined by a supremum over a compact set is universally measurable--- a result of independent interest. The universal $\sigma$-algebra is smaller than the completion of $\cB(\Rset^d)$ with respect to any Borel measure. Thus, in the remainder of the paper, unless otherwise noted, all measures will be Borel measures and the expression $\int S_\e(f)d\QQ$ will be interpreted as the integral of $S_\e(f)$ with respect to the completion of $\QQ$.


\subsection{The $W_\infty$ metric}
In this section, we explain how the integral of a supremum $\int S_\e(f)d\QQ$ can be expressed in terms of a supremum of integrals. We start by defining the Wasserstein-$\infty$ metric. 
\begin{definition}
    Let $\PP,\QQ$ be two finite measures with  $\PP(\Rset^d)=\QQ(\Rset^d)$. A \emph{coupling} is a  positive measure on the product space $\Rset^d\times \Rset^d$ with marginals $\PP,\QQ$. We denote the set of all couplings with marginals $\PP$, $\QQ$ by $\Pi(\PP,\QQ)$. The $\infty$-Wasserstein distance with respect to a norm $\|\cdot\|$ is defined as
    \[W_\infty(\PP,\QQ)=\inf_{\gamma\in \Pi(\PP,\QQ)} \esssup_{(\bx,\bx')\sim \gamma} \|\bx-\bx'\|\]
\end{definition}
In other words, $\PP$, $\QQ$ are within a Wasserstein-$\infty$ distance of $\e$ if there is a coupling $\gamma$ for $\PP$ and $\QQ$ for which $\supp \gamma$ is contained in the set $\Delta_\e=\{(\bx,\bx')\colon \|\bx-\bx'\|\leq \e\}$.

The $\infty$-Wasserstein metric is closely related to the to the operation $S_\e$.
First, we show that $S_\e$ can be expressed as a supremum of integrals over a Wasserstein-$\infty$ ball.
For a measure $\QQ$, we write
\[\Wball \e (\QQ)=\{\QQ'\text{ Borel}: W_\infty(\QQ,\QQ')\leq \e\}.\]
\begin{restatable}{lem}{lemmaSeWinftyequivalence}\label{lemma:S_e_W_infty_equivalence}
     Let $\QQ$ be a finite positive Borel measure and let $f\colon \Rset^d\to \Rset \cup \{\infty\}$ be a Borel measurable function. Then
     \begin{equation}\label{eq:sup_integral_swap}
        \int S_\e(f)d\QQ=\sup_{\substack{\QQ' \in \Wball \e (\QQ)}} \int f d\QQ'
     \end{equation}
\end{restatable}
Lemma~5.1 of \citet{PydiJog2021} proves an analogous statement for sets, namely that $\QQ(A^\e)=\sup_{\QQ'\in \Wball \e(\QQ)} \QQ(A)$, under suitable assumptions on $\QQ$ and $\QQ'$.


Conversely, the $W_\infty$ distance between two probability measures can be expressed in terms of the integrals of $f$ and $S_\e(f)$. Let $C_b(X)$ be the set of continuous bounded functions on the topological space $X$.
\begin{restatable}{lem}{lemmaWinfintegralcharacterization}\label{lemma:W_inf_integral_characterization}
         Let $\PP,\QQ$ be two finite positive Borel measures with $\PP(\Rset^d)=\QQ(\Rset^d)$.
     Then
     \[W_\infty(\PP,\QQ)=\inf_\e\{\e\geq 0 \colon \int h d\QQ\leq \int S_\e(h)d\PP \quad \forall h\in C_b(\Rset^d)\}\]
\end{restatable}

This observation will be central to proving a duality result. See Appendix~\ref{app:W_infty} for proofs of Lemmas~\ref{lemma:S_e_W_infty_equivalence} and~\ref{lemma:W_inf_integral_characterization}.

\section{Main Results and Outline of the Paper}
\subsection{Summary of Main Results}\label{sec:main_results_summary}
Our goal in this paper is to understand properties of the surrogate risk minimization problem $\inf_f R_\phi^\e$.
The starting point for our results is Lemma~\ref{lemma:S_e_W_infty_equivalence}, which implies that $\inf_f R_\phi^\e$ actually involves a $\inf$ followed by a $\sup$: 
    \[\inf_{f\text{ Borel}}R_\phi^\e(f)=\inf_{f\text{ Borel}}\sup_{\substack{\PP_0'\in \Wball \e(\PP_0)\\ \PP_1'\in \Wball \e (\PP_1)}} \int \phi\circ f d\PP_1'+\int \phi \circ -f d\PP_0'.\]
    We therefore obtain a lower bound on $\inf_f R_\phi^\e$ by swapping the $\sup$ and $\inf$ and recalling the definition of $C_\phi^*(\eta)=\inf_\alpha C_\phi(\eta,\alpha)$:
    \begin{align}
        \inf_{f\text{ Borel}}R_\phi^\e(f)
        &\geq \sup_{\substack{\PP_0'\in \Wball \e(\PP_0)\\ \PP_1'\in \Wball \e(\PP_1)}} \inf_{f\text{ Borel}}\int \phi\circ f d\PP_1'+\int \phi \circ -f d\PP_0'\nonumber\\
        &    =\sup_{\substack{\PP_0'\in \Wball \e(\PP_0)\\ \PP_1'\in \Wball \e(\PP_1)}}     \inf_{f\text{ Borel}}\int \frac{d\PP_1'}{d(\PP_0'+\PP_1')} \phi(f)+\left(1-\frac{d\PP_1'}{d(\PP_0'+\PP_1')}\right)\phi(-f) d(\PP_0'+\PP_1')\nonumber\\
    &\geq \sup_{\substack{\PP_0'\in \Wball \e(\PP_0)\\ \PP_1'\in \Wball \e(\PP_1)}} \int C_\phi^*\left(\frac{d\PP_1'}{d(\PP_0'+\PP_1')}\right)d(\PP_0'+\PP_1')\label{eq:weak_duality_last}
    \end{align}
If we define
    \begin{equation}\label{eq:dual_objective_def}
        \dl(\PP_0',\PP_1')=\int C_\phi^*\left(\frac{d\PP_1'}{d(\PP_0'+\PP_1')}\right)d(\PP_0'+\PP_1'),
    \end{equation}
then we have shown
    \begin{equation}\label{eq:weak_duality_original}
        \inf_{f\text{ Borel}} R_\phi^\e(f)\geq \sup_{\substack{\PP_0'\in \Wball \e(\PP_0)\\ \PP_1'\in \Wball \e(\PP_1)}} \dl(\PP_0',\PP_1').    
    \end{equation}
    This statement is a form of weak duality.

	When the surrogate adversarial risk is replaced by the standard adversarial classification risk, 
    \citet{PydiJog2021} proved that the analogue of \eqref{eq:weak_duality_original} is actually an equality, so that strong duality holds for the adversarial classification problem.
    Concretely, by analogy to \eqref{eq:dual_objective_def}, consider
        \[\bar R(\PP_0',\PP_1')=\int C^*\left( \frac{d\PP_1'}{d(\PP_0'+\PP_1')}\right)d(\PP_0'+\PP_1').\]
    Let $\mu$ be the Lebesgue measure and let $\cL_\mu(\Rset^d)$ be the Lebesgue $\sigma$-algebra. Then define
    \begin{equation}\label{eq:tWball_def}
        \tWball \e (\QQ)=\{\QQ'\colon W_\infty(\QQ,\QQ')\leq \e,\QQ' \text{ a measure on $(\Rset^d, \cL_\mu(\Rset^d))$} \}.        
    \end{equation}
\citet{PydiJog2021} show the following.
            \begin{theorem}[{\citealp[Theorem 7.1]{PydiJog2021}}]\label{th:minimax_classification}
        Assume that $\PP_0,\PP_1$ are absolutely continuous with respect to the Lebesgue measure $\mu$. Then
	        \begin{equation}\label{eq:minimax_classification}
	            \inf_{f\text{ Lebesgue}}   R^\e(f)
             =\sup_{\substack{\PP_0'\in \tWball \e(\PP_0)\\ \PP_1'\in \tWball \e (\PP_1) }}  \bar R (\PP_0',\PP_1')    
	        \end{equation}
	        and furthermore equality is attained at some Lebesgue measurable $\hat f$ and $\hat \PP_1,\hat \PP_0$. 

	        Additionally, $\hat \PP_i=\PP_i\circ \varphi_i^{-1}$ for some universally measurable $\varphi_i$ with $\|\varphi_i(\bx)-\bx\|\leq \e$, $\sup_{\|\by-\bx\|\leq \e}\one_{\hat f(\by)\leq 0}=\one_{\hat f(\varphi_1(\bx))\leq 0}$ $\PP_1$-a.e., and $\sup_{\|\by-\bx\|\leq \e}\one_{\hat f(\by)> 0}=\one_{\hat f(\varphi_0(\bx))>0}$ $\PP_0$-a.e.
    \end{theorem}
     This is a foundational result in the theory of adversarial classification, but it leaves an open question crucial in applications:
     Does the strong duality relation extend to surrogate risks and to general measures?
     In this work, we answer this question in the affirmative.
     
    We start by proving the following:
    \begin{restatable}[Strong Duality]{theorem}{thstrongduality}\label{th:strong_duality} 
    Let $\PP_0,\PP_1$ be finite Borel measures. 
    Then
    \begin{equation}
    \label{eq:minimax_surrogate}
	            \inf_{f\text{ Borel}}   \prm(f)
             =\sup_{\substack{\PP_0'\in \Wball \e(\PP_0)\\ \PP_1'\in \Wball \e (\PP_1) }}  \dl (\PP_0',\PP_1').    
	        \end{equation}
        
    \end{restatable}
	When $\e = 0$, we recover the fundamental characterization of the minimum risk for standard (non-adversarial) classification in~\eqref{eq:R_phi_min}.
    Theorem~\ref{th:strong_duality} can be rephrased as
    \begin{equation}\label{eq:minimax_phi_explicit}
        \inf_{f\text{ Borel}} \sup_{\substack{\PP_0'\in \Wball \e(\PP_0)\\ \PP_1'\in \Wball \e (\PP_1) }} \hat R_\phi(f,\PP_0',\PP_1')= \sup_{\substack{\PP_0'\in \Wball \e(\PP_0)\\ \PP_1'\in \Wball \e (\PP_1) }} \inf_{f\text{ Borel}}\hat R_\phi(f,\PP_0',\PP_1')
    \end{equation}
        
    where 
    \[\hat R_\phi(f,\PP_0',\PP_1')=\int \phi(f)d\PP_1'+\int \phi(-f) d\PP_0'\]
                    
                    As discussed in \citet{PydiJog2021}, this result has an appealing game theoretic interpretation:
                    adversarial learning with a surrogate risk can be though of as a zero-sum game between the learner who selects a function $f$ and the adversary who chooses perturbations through $\PP_0'$ and $\PP_1'$. Furthermore, the player to pick first does not have an advantage. 

Additionally, \eqref{eq:minimax_phi_explicit} suggest that training adversarially robust classifiers could be accomplished by optimizing over primal functions $f$ and dual distributions $\PP_0',\PP_1'$ \emph{simultaneously}.

    

		A consequence of Theorem~\ref{th:strong_duality} is the following complementary slackness conditions for optimizers $f^*, \PP_0^*,\PP_1^*$:        
            \begin{restatable}[Complimentary Slackness]{theorem}{thcomplimentaryslackness}
    \label{th:complimentary_slackness}
        The function $f^*$ is a minimizer of $R_\phi^\e$ and $(\PP_0^*,\PP_1^*)$ is a maximizer of $\bar R_\phi$ over the $W_\infty$ balls around $\PP_0$ and $\PP_1$ iff the following hold: 
    \begin{enumerate}[label=\arabic*)]
        \item\label{it:sup_assumed_precondition} 

        \begin{equation}\label{eq:sup_comp_slack}
            \int \phi \circ f^* d\PP_1^*=\int S_\e(\phi(f^*))d\PP_1\quad \text{and} \quad \int \phi \circ -f^* d\PP_0^*=\int S_\e(\phi(f^*))d\PP_0 
        \end{equation}
        \item \label{it:comp_slack_equation} If we define $\PP^*=\PP_0^*+\PP_1^*$ and $\eta^*=d\PP_1^*/d\PP^*$, then 
    \begin{equation}\label{eq:complimentary_slackness_necessary}
        \eta^*(\bx)\phi(f^*(\bx))+(1-\eta^*(\bx))\phi(-f^*(\bx))=C_\phi^*(\eta^*(\bx))\quad \PP^*\text{-a.e.}    
    \end{equation}
    \end{enumerate}

    \end{restatable}
    
    This theorem implies that every minimizer $f^*$ of $\prm$ and every maximizer $(\PP_0^*,\PP_1^*)$ of $\dl$ forms a primal-dual pair. 
    The condition \eqref{eq:sup_comp_slack} states that every maximizer of $\dl$ is an optimal adversarial attack on $f^*$ while the condition \eqref{eq:complimentary_slackness_necessary} states that every minimizer $f^*$ of $\prm$ also minimizes the conditional risk $C_\phi(\eta^*,\cdot)$ under the distribution of optimal adversarial attacks. In other words, $f^*$ minimizes the \emph{standard} surrogate risk with respect to the optimal adversarially perturbed distributions. Explicitly: \eqref{eq:complimentary_slackness_necessary} implies that every minimizer $f^*$ minimizes the loss $\hat R_\phi(f,\PP_0^*,\PP_1^*)=\int C(\eta^*(\bx),f(\bx))d\PP^*$ in a pointwise manner $\PP^*$-a.e. This fact is the relation \eqref{eq:standard_optimal_f} with respect to the measures $\PP_0^*,\PP_1^*$ that maximize the dual $\dl$.



    This interpretation of Theorems~\ref{th:strong_duality} and~\ref{th:complimentary_slackness} shed light on the phenomenon of transfer attacks. 
  These theorems suggests that adversarial examples are a property of the data distribution rather than a specific model. Later results in the paper even show that there are maximizers of $\dl$ that are independent of the choice of loss function $\phi$ (see Lemma~\ref{lemma:minimizer_construction}).  Theorem~\ref{th:complimentary_slackness} specifically states that every maximizer of $\dl$ is actually an optimal adversarial attack on \emph{every} minimizer of $\prm$. Notably, this statement is \emph{indepent of the choice of minimizer of $\prm$.} Because neural networks are highly expressive model classes, one would hope that some neural net could achieve adversarial error close to $\inf_f R^\e_{\phi}(f)$. If $f^*$ is a minimizer of $R_\phi^\e$ and $g$ is a neural net with $R^\e_{\phi}(g)\approx R^\e_\phi(f^*)$, one would expect that an optimal adversarial attack against $f^*$ would be a successful attack on $g$ as well. Notice that in this discussion, the adversarial attack is independent of the neural net $g$. A method for calculating these optimal adversarial attacks is an open problem.

    Finally, to demonstrate that Theorem~\ref{th:complimentary_slackness} and the preceding discussion is non-vacuous, we prove the existence of primal and dual optimizers along with results that elaborate on their structure. 
 
    \begin{restatable}{theorem}{thdualexistence}\label{th:dual_existence}
    Let $\phi$ be a lower-semicontinuous loss function. Then there exists a maximizer $(\PP_0^*,\PP_1^*)$ to $\bar R_\phi$ over the set $\Wball \e(\PP_0)\times \Wball \e(\PP_1)$. 
\end{restatable} 

Theorem~3.5 of \citep{Jylha15} implies that when the norm $\|\cdot\|$ is strictly convex and $\PP_0,\PP_1$ are absolutely continuous with respect to Lebesgue measure, the optimal $\PP_0^*,\PP_1^*$ of Theorem~\ref{th:dual_existence} are induced by a transport map. Corollary~3.11 of \citep{Jylha15} further implies that these transport maps are continuous a.e. with respect to the Lebesgue measure $\mu$. As the $\ell_\infty$ metric is commonly used in practice, whether there exist maximizers of the dual of this type for non-strictly convex norms remains an attractive open problem.

In analogy with \eqref{eq:standard_optimal_f} and \eqref{eq:minimizer_from_eta} one would hope that due to the complimentary slackness condition \eqref{eq:complimentary_slackness_necessary}, one could define a minimizer in terms of the conditional $\eta^*(\bx)$. Notice, however, that as this quantity is only defined $\PP^*$-a.e., verifying the other complimentary slackness condition \eqref{eq:sup_comp_slack} would be a challenge. 
To circumvent this issue, we construct a function $\hat \eta : \Rset^d\to [0,1]$, defined on all of $\Rset^d$, to which we can apply \eqref{eq:minimizer_from_eta}.
Concretely, we show that $\alpha_\phi(\hat \eta(\bx))$ is always a minimizer of $\prm$, with $\alpha_\phi$ as defined in \eqref{eq:alpha_phi_definition}.
\begin{restatable}{theorem}{thexistenceprimal}\label{th:existence_primal}
      There exists a Borel function $\hat \eta:(\supp \PP)^\e\to \ov [0,1]$ for which $f^*(\bx)=\alpha_\phi(\hat \eta(\bx))$ is a minimizer of $\prm$ for any $\phi$ with $\alpha_\phi$ as in defined in \eqref{eq:alpha_phi_definition}. In particular, there exists a Borel minimizer of $\prm$.
\end{restatable}
    In fact, we show that $\hat \eta$ is a version of the conditional derivative $d\PP_1^*/d\PP^*$, where $\PP_0^*,\PP_1^*$ are the measures which maximize $\dl$ independently of the choice of $\phi$ (see Lemma~\ref{lemma:gamma_supp}), as described in the discussion preceding Theorem~\ref{th:dual_existence}. The fact that the function $\hat \eta$ is independent of the choice of loss $\phi$ suggests that the minimizer of $\prm$ encodes some fundamental quality of the distribution $\PP_0,\PP_1$.

    Simultaneous work \citep{LiTelgarsky2023achieving} also proves the existence of a minimizer to the primal $\prm$ along with a statement on the structure of this minimizer. Their approach leverages prior results on the adversarial Bayes classifier to construct a minimizer to the adversarial surrogate risk.
    \subsection{Outline of Main Argument}\label{sec:proof_summary}
    
    The central proof strategy of this paper is to apply the Fenchel-Rockafellar duality theorem. This classical result relates the infimum of a convex functional with the supremum of a concave functional. One can argue that $\dl$ is concave (Lemma~\ref{lemma:Xi_star_pre} below); however, the primal $R_\phi^\e$ is not convex for non-convex $\phi$. Thus the Fenchel-Rockafellar theorem is applied to a convex relaxation $\Theta$ of the primal $\prm$.

    The remainder of the paper then argues that minimizing $\Theta$ is equivalent to minimizing $\prm$. Notice that the Fenchel-Rockafellar theorem actually implies the existence of dual maximizers. We show that that dual maximizers of $\bar R_\psi$ for $\psi(\alpha)=e^{-\alpha}$ satisfy certain nice properties that are \emph{independent} of the loss $\psi$. These properties then allow us to construct the function $\hat \eta$ present in Theorem~\ref{th:existence_primal} in addition to minimizers of $\Theta$ from the dual maximizers of $\bar R_\psi$, for any loss $\phi$. The construction of these minimizers guarantee that they minimize $\prm$ in addition to $\Theta$.
    \subsection{Paper Outline}
         Section~\ref{sec:duality_fundamental} proves strong duality and complimentary slackness theorems for $\dl$ and $\Theta$, the convex relaxation of $\prm$. Next, in Section~\ref{sec:primal_existence}, a version of the complimentary slackness result is used to prove the existence of minimizers to $\Theta$. Subsequently, Section~\ref{sec:reduction} shows the equivalence between $\Theta$ and $\prm$. 

        Appendix~\ref{app:meas} proves Theorem~\ref{th:meas} and further discusses universal measurability. Next, Appendix~\ref{app:W_infty} proves all the results about the $W_\infty$-norm used in this paper. Appendix~\ref{app:C_phi_minimizers} then defines the function $\alpha_\phi$  which is later used in the proof of several results. Appendices~\ref{app:dl_C_b}, ~\ref{app:duality_all},~\ref{app:complimentary_slackness}, and~\ref{app:S_e_liminf_limsup_swap} contain technical deferred proofs.

\section{A Duality Result for $\Theta$ and $\dl$}\label{sec:duality_fundamental} 
\subsection{Strong Duality}

The fundamental duality relation of this paper follows from employing the Fenchel-Rockafellar theorem in conjunction with the Riesz representation theorem, stated below for reference. See e.g. \citep{TopicsInOptimalTransportVillani} for more on this result.
        \begin{theorem}[Fenchel-Rockafellar Duality Theorem] \label{th:fenchel_rockafellar}
    Let $E$ be a normed vector space $E^*$ its topological dual and $\Theta,\Xi$ two convex functionals on $E$ with values in $\Rset\cup \{\infty\}$. Let $\Theta^*,\Xi^*$ be the Legendre-Fenchel transforms of $\Theta,\Xi$ respectively. Assume that there exists $z_0\in E$ such that
    \[\Theta(z_0)< \infty,\Xi(z_0)< \infty\]
    and that $\Theta$ is continuous at $z_0$. Then
    \begin{equation}\label{eq:fenchel_rockafellar}
        \inf_{z\in E}[\Theta(z)+\Xi(z)]=\sup_{z^*\in E^*}[-\Theta^*(z^*)-\Xi^*(-z^*)]
    \end{equation}
    and furthermore, the supremum on the right hand side is attained.
\end{theorem}

    Let $\cM(X)$ be the set of finite signed Borel measures on a space $X$ and recall that $C_b(X)$ is the set of bounded continuous functions on the space $X$.
\begin{theorem}[Riesz Representation Theorem]\label{th:reisz_representation_theorem}
    Let $K$ be any compact subset of $\Rset^d$. Then the dual of $C_b(K)$ is $\cM(K)$. 
\end{theorem}
See Theorem~1.9 of \citep{TopicsInOptimalTransportVillani} and result~7.17 of \citep{folland} for more details.

Notice that in the Fenchel-Rockafellar theorem, the left-hand side of \eqref{eq:fenchel_rockafellar} is convex while the right-hand side is concave. However, when $\phi$ is non-convex, $\prm$ is not convex. In order to apply the Fenchel-Rockafellar theorem, we will relax the primal $\prm$ will to a convex functional $\Theta$. 

We define $\Theta$ as
 \begin{equation}\label{eq:Theta_def}
        \Theta(h_0,h_1)=\int S_\e(h_1)d\PP_1+\int S_\e(h_0)d\PP_0
    \end{equation}
which is convex in $h_0,h_1$ due to the sub-additivity of the supremum operation. Notice that one obtains $\Theta$ from $\prm$ by replacing $\phi \circ f$ with $h_1$ and $\phi \circ -f$ with $h_0$.

The functional $\Xi$ will be chosen to restrict $h_0,h_1$ in the hope that at the optimal value, $h_1=\phi(f)$ and $h_0=\phi(-f)$ for some $f$. Notice that if $h_1=\phi(f)$, $h_0=\phi(-f)$ then for all $\eta\in [0,1]$, 
\[\eta h_1(\bx)+(1-\eta)h_0=\eta\phi(f))+(1-\eta)\phi(-f)\geq C_\phi^*(\eta).\]
Thus we will optimize $\Theta$ over the set of functions $S_\phi$ defined by
    \begin{equation}\label{eq:S_def}
        S_\phi=\left\{ \begin{aligned}
        &(h_0,h_1)\colon h_0,h_1\colon K^\e \to \ov \Rset\text{ Borel, }0\leq h_0,h_1 \text{ and for}\\
        &  \text{all $\bx\in \Rset^d$,  $\eta\in[0,1]$, } \eta h_0(\bx)+(1-\eta)h_1(\bx)\geq C_\phi^*(\eta)
        \end{aligned}\right\}
    \end{equation}
where $K=\supp(\PP_0\cup \PP_1)$. 
(Notice that the definition of $S_\e(g)$ in \eqref{eq:S_e_def} assumes that the domain of $g$ must include $\ov{B_\e(\bx)}$. Thus in order to define the integral $\int S_\e(h)d\QQ$, the domain of $h$ must include $(\supp \QQ)^\e$.)
Thus we define $\Xi$ as
\begin{equation}\label{eq:Xi_def}    \Xi(h_0,h_1)=    \begin{cases}
            0 &\text{if }(h_0,h_1)\in S_\phi\\
            +\infty &\text{otherwise}
    \end{cases}
    \end{equation}

The following result expresses $\dl$ as an infimum of linear functionals continuous with respect to the weak topology on probability measures. 
This lemma will assist in the computation of $\Xi^*$. In the remainder of this section, $\cM_+(S)$ will denote the set of positive finite Borel measures on a set $S$.
 \begin{restatable}{lem}{lemmaxistarpre}\label{lemma:Xi_star_pre}
    Let $K\subset \Rset^d$ be compact, $E=C_b(K^\e)\times C_b(K^\e)$, and $\PP_0',\PP_1'\in \cM_+(K^\e)$. Then
    \begin{equation}\label{eq:dl_inf_representation}
        \inf_{\substack{(h_0,h_1)\in S_\phi\cap E}} \int h_1d\PP_1'+\int h_0d\PP_0'=\dl(\PP_0',\PP_1')    
    \end{equation}
    
    Therefore, $\dl$ is concave and upper semi-continuous on $\cM_+(K^\e)\times \cM_+(K^\e)$ with respect to the weak topology on probability measures. 
\end{restatable}
We sketch the proof and formally fill in the details in Appendix~\ref{app:dl_C_b}. Let $\PP'=\PP_0'+\PP_1'$, $\eta'=d\PP_1'/d\PP'$. Then 
\[ \int h_1d\PP_1'+\int h_0d\PP_0'=\int \eta' h_1+(1-\eta')h_0d\PP'\]
 Clearly, the inequality $\geq$ holds because $\eta'  h_1+(1-\eta')h_0\geq C_\phi^*(\eta')$ for all $(h_0,h_1)\in S_\phi$.
 Equality is achieved at $h_1=\phi(\alpha_\phi(\eta')),h_0=\phi(-\alpha_\phi(\eta'))$, with $\alpha_\phi$ as in \eqref{eq:alpha_phi_definition}.
 However, these functions may not be continuous. In Appendix~\ref{app:dl_C_b}, we show that $h_0,h_1$ can be approximated arbitrarily well by elements of $S_\phi \cap E$.

\begin{lemma}\label{lemma:duality_compact}
     Let $\phi$ be a non-increasing, lower semi-continuous loss function and let $\PP_0,\PP_1$ be compactly supported finite Borel measures. Let $S_\phi$ be as in \eqref{eq:S_def}.

     Then 
     \begin{equation}\label{eq:duality_first}
    \inf_{(h_0,h_1)\in S_\phi} \Theta(h_0,h_1)=\sup_{\substack{\PP_0'\in \Wball\e (\PP_0)\\ \PP_1'\in \Wball\e(\PP_1)}} \bar R_\phi(\PP_0',\PP_1')
     \end{equation}
     Furthermore, there exist $\PP_0^*,\PP_1^*$ which attain the supremum.

\end{lemma}

\begin{proof}
We will show a version of \eqref{eq:duality_first} with the infimum taken over $S_\phi \cap E$, and then argue that the same claim holds when the infimum is taken over $S_\phi$.

         Notice that if $h_0$, $h_1$ are continuous, then $S_\e(h_0)$, $S_\e(h_1)$ are also continuous and $\int S_\e(h_0)d\QQ$, $\int S_\e(h_1)d\QQ$ are well-defined for every \emph{Borel} $\QQ$. Hence 
         we assume that $\PP_0$, $\PP_1$ are Borel measures rather than their completion.

        Let $K=\supp (\PP_0+\PP_1)$. We will apply the Fenchel-Rockafellar Duality Theorem to the functionals given by \eqref{eq:Theta_def} and \eqref{eq:Xi_def}
    on the vector space $E=C_b(K^\e)\times C_b(K^\e)$ equipped with the sup norm. By the Riesz representation theorem, dual of the space $E$ is $E^*=\cM(K^\e)\times \cM(K^\e)$. 
    
    To start, we argue that the Fenchel-Rockafellar duality theorem applies to these functionals. First, notice that if $(h_0,h_1)\in E$, then both $h_0,h_1$ are bounded so $\Theta(h_0,h_1)<\infty$. 
    Furthermore, $\Theta$ is convex due to the subadditivity of supremum and $\Xi$ is convex because the constraint $h_0(\bx)+(1-\eta)h_1(\bx)\geq C_\phi^*(\eta)$ is linear in $h_0$ and $h_1$. 
    Furthermore, $\Theta$ is continuous on $E$ because $\Theta$ is convex and bounded and $E$ is open, see Theorem~2.14 of \citep{ConvexityAndOptimizationInBanachSpaces}.
    


    Because the constant function $(C_\phi^*(1/2),C_\phi^*(1/2))$ is in $S_\phi$, $\Xi$ is not identically $\infty$ and therefore the Fenchel-Rockafellar theorem applies.

    Clearly $\inf_{E} \Theta(h_0,h_1)+\Xi(h_0,h_1)$ reduces 
    to the left-hand side of \eqref{eq:duality_first}. 
    
    We now compute the dual of $\Xi$. Lemma~\ref{lemma:Xi_star_pre} implies that
    \begin{align}
        &-\Xi^*(-\PP_0',-\PP_1')=-\sup_{(h_0,h_1)\in S_\phi\cap E } -\int h_0d\PP_0'-\int h_1d\PP_1'\label{eq:Xi_star_computation}\\
        &=\begin{cases}
             \dl(\PP_0',\PP_1')
        &\text{if }\PP_i'\geq 0\\
        -\infty &\text{otherwise}
        \end{cases}\nonumber
    \end{align}
    This computation implies that the term $-\Xi^*(-\PP_0',-\PP_1')$ present in the Fenchel-Rockafellar Theorem is not $-\infty$ iff $\PP_0',\PP_1'$ are positive measures. Next, notice that because $\Theta(h_0,h_1)<+\infty$ for all $(h_0,h_1)\in E$, $-\Theta^*(\PP_0',\PP_1')$ is never $+\infty$. Therefore, it suffices to compute $\Theta^*$ for positive measures $\PP_0',\PP_1'$. Lemma~\ref{lemma:W_inf_integral_characterization} implies that for positive measures $\PP_0',\PP_1'$,
    \begin{align*}
        &\Theta^*(\PP_0',\PP_1')=\sup_{h_0,h_1\in C_0(K^\e)} \int h_1d\PP_1'+\int h_0d\PP_0'-\left(\int S_\e(h_0)d\PP_0+\int S_\e(h_1)d\PP_1\right)\\
        &=\sup_{h_1\in  C_0(K^\e)} \left(\int h_1d\PP_1'-\int S_\e(h_1)d\PP_1\right)+\sup_{h_0 \in C_0(K^\e)}\left(\int h_0d\PP_0'-\int S_\e(h_0)d\PP_0+\right)\\
        &=\begin{cases}
            0 &\PP_0',\PP_1'\text{ positive measures, with } W_\infty(\PP_0',\PP_0)\leq \e \text{ and }W_\infty(\PP_1',\PP_1)\leq \e\\
            +\infty &\PP_0',\PP_1'\text{ positive measures, with either }W_\infty(\PP_0',\PP_0)> \e\text{ or }W_\infty(\PP_1',\PP_1)> \e
        \end{cases}
    \end{align*}

    Therefore
    \[\sup_{\PP_0',\PP_1'\in \cM(K^\e)} -\Theta(\PP_0',\PP_1')-\Xi(-\PP_0',-\PP_1')=\sup_{\substack{\PP_0'\in \Wball \e(\PP_0)\\\PP_1'\in \Wball \e(\PP_1)}} \bar R_\phi(\PP_0',\PP_1')\]
    and furthermore there exist measures $\PP_0^*,\PP_1^*$ maximizing the dual problem. Therefore
    the Fenchel-Rockafellar Theorem implies that
    \[\inf_{(h_0,h_1)\in S_\phi} \Theta(h_0,h_1)\leq \inf_{(h_0,h_1)\in S_\phi \cap E} \Theta(h_0,h_1)=\sup_{\substack{\PP_0' \in \Wball \e (\PP_0)\\ \PP_1'\in \Wball \e(\PP_1)}}\dl(\PP_0',\PP_1')\]
    The opposite inequality follows from the weak duality argument presented in \eqref{eq:weak_duality_original} in Section~\ref{sec:main_results_summary}. See Lemma~\ref{lemma:weak_duality_borel} of Appendix~\ref{app:duality_all} for a full proof.
\end{proof}
    Note that this proof does not easily extend to an unbounded domain $X$: for a non-compact space, the dual of $C_b(X)$ is much larger than $\cM(X)$, and thus a naive application of the Fenchel-Rockafellar Theorem would result in a different right-hand side than \eqref{eq:duality_first}. On the other hand, the Reisz representation theorem for an unbounded domain $X$ states that the dual of $C_0(X)$ is $\cM(X)$, where $C_0(X)$ is the set of continuous bounded functions vanishing at $\infty$. At the same time, if $h_0,h_1\in C_0(X)$, then $\eta h_1(\bx)+(1-\eta)h_0(\bx)$ becomes arbitrarily small for large $\bx$ so the constraint $\eta h_1(\bx)+(1-\eta)h_0(\bx)\geq C_\phi^*(\eta)$ cannot hold for all $\eta$. Thus if $K$ is unbounded, $S_\phi\cap C_0(X)=\emptyset$ and the functional $\Xi$ would be $+\infty$ everywhere on $C_0(X)$, precluding and application of the Fenchel-Rockafellar Theorem.

However, Lemma~\ref{lemma:duality_compact} can be extended to distributions with arbitrary support via a simple approximation argument.
By Lemma~\ref{lemma:duality_compact}, the strong duality result holds for the distributions defined by $\PP_0^n=\PP_0\big|_{\ov{B_n(\zero)}}, \PP_1^n=\PP_1\big|_{\ov{B_n(\zero)}}$. One then shows strong duality by computing the limit of the primal and dual problems as $n\to \infty$.
We therefore obtain the following Lemma, which is proved formally in Appendix~\ref{app:duality_all}.
 \begin{restatable}{lem}{lemmadualitythetaall}\label{lemma:duality_theta_all}
     Let $\phi$ be a non-increasing, lower semi-continuous loss function and let $\PP_0,\PP_1$ be finite Borel measures supported on $\Rset^d$. Let $S_\phi$ be as in \eqref{eq:S_def}. Then
     \begin{equation*}
    \inf_{(h_0,h_1)\in S_\phi} \Theta(h_0,h_1)=\sup_{\substack{\PP_0'\in \Wball\e (\PP_0)\\ \PP_1'\in \Wball\e(\PP_1)}} \bar R_\phi(\PP_0',\PP_1')
     \end{equation*}
        Furthermore, there exist $\PP_0^*,\PP_1^*$ which attain the supremum.
\end{restatable}

\subsection{Complimentary Slackness}

Using a standard argument, strong duality (Lemma~\ref{lemma:duality_theta_all}) allows us to prove a version of the complimentary slackness theorem. 

\begin{restatable}{lem}{lemmacomplimentaryslacknessrestricted} \label{lemma:complimentary_slackness_restricted} Assume that $\PP_0,\PP_1$ are compactly supported.
            The functions $h_0^*,h_1^*$ minimize $\Theta$ over $S_\phi$ and $(\PP_0^*,\PP_1^*)$ maximize $\bar R_\phi$ over $\Wball \e(\PP_0)\times \Wball \e (\PP_1)$ iff the following hold: 
    \begin{enumerate}[label=\arabic*)]

        \item

        \begin{equation}\label{eq:sup_comp_slack_restricted}
            \int h_1^* d\PP_1^*=\int S_\e(h_1^*)d\PP_1\quad \text{and} \quad \int h_0^* d\PP_0^*=\int S_\e(h_0^*)d\PP_0 
        \end{equation}
        \item If we define $\PP^*=\PP_0^*+\PP_1^*$ and $\eta^*=d\PP_1^*/d\PP^*$, then
    \begin{equation}\label{eq:complimentary_slackness_restricted_necessary}
        \eta^*(\bx)h_1^*(\bx)+(1-\eta^*(\bx))h_0^*(\bx)=C_\phi^*(\eta^*(\bx))\quad \PP^*\text{-a.e.}    
    \end{equation}
    \end{enumerate}
\end{restatable}

This lemma is proved in Appendix~\ref{app:complimentary_slackness}. Theorem~\ref{th:complimentary_slackness} will later follow from this result. 

To show that Lemma~\ref{lemma:complimentary_slackness_restricted} is non-vacuous, one must prove that there exist minimizers to $\Theta$ over $S_\phi$, which we delay to Sections~\ref{sec:primal_existence} and~\ref{sec:reduction}. Notice that the application of the Fenchel-Rockafellar Theorem in Lemma~\ref{lemma:duality_compact} actually implies the existence of dual maximizers in the case of compactly supported $\PP_0,\PP_1$.

In fact, the complimentary slackness conditions hold approximately for any maximizer of $\dl$ and any minimizing sequence of $\Theta$. This result is essential for proving the existence of minimizers to $\Theta$.

	\begin{lemma}\label{lemma:approximate_complimentary_slackness}
		Let $(h_0^n,h_1^n)$ be a minimizing sequence for $\Theta$ over $S_\phi$: $\lim_{n\to \infty} \Theta(h_0^n,h_1^n)=\inf_{(h_0,h_1)\in S_\phi} \Theta(h_0,h_1)$. Then for any maximizer of the dual problem $(\PP_0^*,\PP_1^*)$, the following hold:
    \begin{enumerate}[label=\arabic*)]
        \item \label{it:h_0,1_limit}        \begin{equation}\label{eq:h_0,1_limit}
			\lim_{n\to \infty} \int S_\e(h_0^n)d\PP_0-\int h_0^nd\PP_0^*=0,\quad 			\lim_{n\to \infty} \int S_\e(h_1^n)d\PP_1-\int h_1^nd\PP_1^*=0		
		\end{equation}
        \item  \label{it:C_phi_limit}If we define $\PP^*=\PP_0^*+\PP_1^*$ and $\eta^*=d\PP_1^*/d\PP^*$
        \begin{equation}\label{eq:C_phi_limit}
			\lim_{n\to \infty} \int \eta^*h_1^n+(1-\eta^*)h_0^n-C_\phi^*(\eta^*)d\PP^*=0
		\end{equation}
    \end{enumerate}

	\end{lemma}
	\begin{proof}
		 Let 
		 \[m=\inf_{(h_0,h_1)\in S_\phi }\Theta(h_0,h_1).\]
		 Then the fact that $(h_0^n,h_1^n)\in S_\phi$ and		 
		  the duality result (Lemma~\ref{lemma:duality_theta_all}) implies
		  \begin{equation}\label{eq:basic_theta_lower}
		  		\int h_1^nd\PP_1^*+\int h_0^nd\PP_0^*=\int \eta^*h_1^n+(1-\eta^*)h_0^nd\PP^*\geq \int C_\phi^*(\eta^*)d\PP^*=m
		  \end{equation}

		 		 Now pick $\delta>0$ and an $N$ for which $n\geq N$ implies that 	
		 $\Theta(h_0^n,h_1^n)\leq m+\delta$. Then
		 \[m+\delta\geq \int S_\e(h_1^n)d\PP_1+\int S_\e(h_0^n)d\PP_0\geq 	\int \eta^*h_1^n+(1-\eta^*)h_0^nd\PP^*\geq m.\]
		 Subtracting $m=\int C_\phi^*(\eta^*)d\PP^*$ from this inequality results in
   \begin{equation}\label{eq:key_linear}
      \delta\geq \int \eta^*h_1^n+(1-\eta^*)h_0^n d\PP^* -\int C_\phi^*(\eta^*)d\PP^* \geq 0 
   \end{equation}
		 
		 which implies \eqref{eq:C_phi_limit}.	 
		Next, \eqref{eq:basic_theta_lower} further implies
		 \begin{equation}\label{eq:h_0^n_h_1^n_Theta_star_inequality}
			m-\int h_1^nd\PP_1^*+\int h_0^nd\PP_0^*\leq 0
		 \end{equation}
		 
		  Now this inequality implies
		  \begin{align*}
				&\delta\geq \delta+m-\left(\int h_1^nd\PP_1^*+\int h_0^nd\PP_0^* \right)\geq \Theta(h_1^n,h_0^n)-\left( \int h_1^nd\PP_1^*+\int h_0^nd\PP_0^* \right)\\
				&\geq \left(\int S_\e(h_1^n)d\PP_1+\int S_\e(h_0^n)d\PP_0\right)-\left(\int h_1^nd\PP_1^*+\int h_0^nd\PP_0^* \right)  	
		  \end{align*}

		However, Lemma~\ref{lemma:S_e_W_infty_equivalence} implies that both $\int S_\e(h_1^n)d\PP_1-\int h_1^nd\PP_1^*,\int S_\e(h_0^n)d\PP_0-\int h_0^nd\PP_0^*$ are positive quantities. 
		Therefore, we have shown that for any $\delta>0$, there is an $N$ for which $n\geq N$ implies that 
		\[\delta> \int S_\e(h_1^n)d\PP_1-\int h_1^nd\PP_1^*\geq 0\quad \text{and}\quad \delta>\int S_\e(h_0^n)d\PP_0-\int h_0^nd\PP_0^*\geq 0\]
		which implies \eqref{eq:h_0,1_limit}.
		
	\end{proof}
    An analogous approximate complimentary slackness result typically holds in other applications of the Fenchel-Rockafellar theorem. Consider a convex optimization problem which can be written as $\inf_x \Theta(x)+\Xi(x)$ in such a way that the Fenchel-Rockafellar theorem applies and for which $\Xi$ and $\Theta^*$ are indicator functions of the convex sets $C_P,C_D$ respectively. Then the Fenchel-Rockafellar Theorem states that 
    \begin{equation}\label{eq:linear_duality}
        \inf_{x\in C_P}\Theta(x)=\inf_{x\in C_p}\sup_{y\in C_D} \langle y, x\rangle =\sup_{y\in C_D} \inf_{x\in C_P} \langle y, x\rangle =\sup_{y\in C_D} \Xi^*(y)   
    \end{equation}

    Let $y^*$ be a maximizer of the dual problem and let $m$ be the minimal value of $\Theta$ over $C_P$. If $x_k$ is a minimizing sequence of $\Theta$, then for $\delta>0$ and sufficiently large $k$, $\delta+m>\Theta(x_k)$ and thus by \eqref{eq:linear_duality},
    \begin{equation}\label{eq:general_duality_approx_slack_arugment}
        m+\delta>\Theta(x_k)=\sup_{y\in C_p} \langle y, x_k\rangle\geq \langle y^*, x_k\rangle \geq \inf_{x\in C_D} \langle y^*,x\rangle =\inf_{x \in C_D} \Xi^*(x)=m
    \end{equation}

    and therefore $\lim_{k\to \infty} \langle y^*,x_k\rangle =m$. Condition \eqref{eq:complimentary_slackness_restricted_necessary} is this statement adapted to the adversarial learning problem. Furthermore, subtracting $\Theta(x_k)$ from \eqref{eq:general_duality_approx_slack_arugment} and taking the limit $k\to \infty$ results in $\lim_{k\to\infty} \Theta(x_k)-\langle y^*,x_k\rangle=0$. In our adversarial learning scenario, this statement is equivalent to the conditions in \eqref{eq:h_0,1_limit} due to Lemma~\ref{lemma:S_e_W_infty_equivalence}.
    Furthermore, just like the standard complimentary slackness theorems, the relations $\lim_{k\to \infty} \langle y^*,x_k\rangle=m$, $\lim_{k\to \infty} \Theta(x_k)-\langle y^*,x_k\rangle =0$ imply that $x_k$ is a minimizing sequence for $\Theta$.
    
     In the classical proof of the Kantorovich duality, one can choose $\Theta,\Xi$ of a form similar to the discussion above, see for instance Theorem~1.3 of \cite{TopicsInOptimalTransportVillani}. Using an argument similar to \eqref{eq:general_duality_approx_slack_arugment}, one can prove approximate complimentary slackness for the Kantorovich problem called the quantitative Knott-Smith criteria, see Theorems~2.15,~2.16 of \cite{TopicsInOptimalTransportVillani} for further discussion.

\section{Existence of minimizers of $\Theta$ over $S_\psi$}\label{sec:primal_existence}
After proving the existence of maximizers to the dual problem, we can now use the approximate complimentary slackness conditions to prove the existence of minimizers to the primal. The exponential loss $\psi$ has certain properties that make it particularly easy to study:

            \begin{restatable}{lem}{lemmaexponentialloss}\label{lemma:exponential_loss}
		Let $\psi(\alpha)=e^{-\alpha}$. Then $C_\psi^*(\eta)=2\sqrt{\eta(1-\eta)}$ and $\alpha_\psi(\eta)=1/2\log(\eta/1-\eta)$ is the unique minimizer of $C_\psi(\eta,\cdot)$, with $\alpha_\psi(0),\alpha_\psi(1)$ interpreted as $-\infty$, $+\infty$ respectively.
		Furthermore, $\partial C_\psi^*(\eta)$ is the singleton $\partial C_\psi^*(\eta)=\{\psi(\alpha_\psi(\eta))-\psi(-\alpha_\psi(\eta))\}$.
	\end{restatable}

See Appendix~\ref{app:exponential_loss} for a proof. The existence of minimizers of $\Theta$ for the exponential loss then follows from properties of $C_\psi$. Let $(h_0^n,h_1^n)$ be a minimizing sequene of $\dl$. Because the function $C_\psi$ is strictly concave, one can use the condition \eqref{eq:C_phi_limit} to show that there is a subsequence $\{n_k\}$ along which $h_0^{n_k}(\bx')$, $h_1^{n_k}(\bx')$ converge $\PP_0^*,\PP_1^*$-a.e. respectively. Due to \eqref{eq:h_0,1_limit}, $S_\e(h_0^{n_k})(\bx)$, $S_\e(h_1^{n_k})$ also converge $\PP_0,\PP_1$-a.e. respectively along this subsequence. This observation suffices to show the existence of functions that minimize $\Theta$ over $S_\psi$.
  
The first step of this proof is to formalize this argument for sequences in $\ov \Rset$.
        \begin{restatable}{lem}{lemmaanbn}\label{lemma:a_n_b_n}
		Let $(a_n,b_n)$ be a sequence for which $a_n,b_n\geq 0$ and 
		\begin{equation}\label{eq:eta_all_ineq}
			\eta a_n+(1-\eta)b_n\geq C_\psi^*(\eta)\text{ for all }\eta\in [0,1]	
		\end{equation}
		and 
		\begin{equation}\label{eq:eta_0_limit}
			\lim_{n\to \infty} \eta_0a_n+(1-\eta_0)b_n=C_\psi^*(\eta_0)
		\end{equation} for some $\eta_0$. Then $\lim_{n\to \infty} a_n=\psi(\alpha_\psi(\eta_0))$ and $\lim_{n\to \infty} b_n=\psi(-\alpha_\psi(\eta_0))$.
	\end{restatable}
	Notice that if $\eta a+(1-\eta)b \geq C_\psi^*(\eta)$ and $\eta_0 a+(1-\eta_0) b=C_\psi^*(\eta_0)$, then this lemma implies that $a=\psi(\alpha_\psi(\eta_0))$ and $b=\psi(-\alpha_\psi(\eta_0))$.

    To prove Lemma~\ref{lemma:a_n_b_n}, we show that all convergent subsequences of $\{a_n\}$ and $\{b_n\}$ must converge to $a$ and $b$ that satisfy $\eta_0 a+(1-\eta_0)b=C_\phi^*(\eta_0)$ and $a-b \in \partial C_\psi^*(\eta_0)$. As the set $\partial C_\psi^*(\eta_0)$ is a singleton, the values $a=\psi(\alpha_\psi(\eta_0))$ and $b=\psi(\alpha_\psi(\eta_0))$ uniquely solve these equations for $a$ and $b$. Therefore the sequences $\{a_n\}$ and $\{b_n\}$ must converge to $a$ and $b$ as well. See Appendix~\ref{app:a_n_b_n} for a formal proof. This result applied to a minimizing sequence of $\Theta$ allows one to find a subsequence with certain convergence properties.

	\begin{lemma}\label{lemma:S_e_limit}
		Let $(h_0^{n},h_1^{n})$ be a minimizing sequence of $\Theta$ over $S_\psi$. 
		Then there exists a subsequence $n_k$ for which $S_\e(h_1^{n_k})$, $S_\e(h_0^{n_k})$ converge $\PP_1$, $\PP_0$-a.e. 
		respectively. 
	\end{lemma}
	\begin{proof}
		Let $\PP_0^*,\PP_1^*$ be maximizers of the dual problem. Let $\gamma_i$ be the coupling between $\PP_i,\PP_i^*$ with $\supp \gamma_i\subset \Delta_\e$. 
		
		Then \eqref{eq:C_phi_limit} of Lemma~\ref{lemma:approximate_complimentary_slackness} implies that 
		\[\lim_{n\to \infty} \int \eta^*(\bx')h_1^n(\bx')+(1-\eta^*(\bx'))h_0^n(\bx') -C_\psi(\eta^*(\bx')) d(\gamma_1+\gamma_0)(\bx,\bx')=0\]
		and \eqref{eq:h_0,1_limit} implies that 
		\[\lim_{n\to \infty} \int S_\e(h_1^n)(\bx)-h_1^n(\bx') d\gamma_1(\bx,\bx')=0,\quad \lim_{n\to \infty} \int S_\e(h_0^n)(\bx)-h_0^n(\bx') d\gamma_0(\bx,\bx')=0\]
		Recall that on a bounded measure space, $L^1$ convergence implies a.e. convergence along a subsequence (see Corollary~2.32 of \citep{folland}). Thus one can pick a subsequence  $n_k$ along which 
		\begin{equation}\label{eq:C_psi_eta_star_rel}
			\lim_{k\to \infty} \eta^*(\bx')h_1^{n_k}(\bx')+(1-\eta^*(\bx'))h_0^{n_k}(\bx') -C_\psi(\eta^*(\bx'))=0
		\end{equation}
		 $\gamma_1+\gamma_0$-a.e. 
		and
		\begin{equation}\label{eq:sup_eta_rel}
			\lim_{k\to \infty} S_\e(h_1^{n_k})(\bx)-h_1^{n_k}(\bx')=0,\quad \lim_{k\to \infty} S_\e(h_0^{n_k})(\bx)-h_0^{n_k}(\bx')=0
		\end{equation}
		$\gamma_1$, $\gamma_0$-a.e. respectively.
		
		Furthermore, $\eta  h_1^n+(1-\eta)h_0^n \geq C_\psi^*(\eta)$ for all $\eta\in [0,1]$. Thus \eqref{eq:C_psi_eta_star_rel} and Lemma~\ref{lemma:a_n_b_n} imply that  $h_{n_k}^1$ converges to $\psi(\alpha_\psi(\eta^*))$ and $h_{n_k}^0$ converges to $\psi(-\alpha_\psi(\eta^*))$ $\gamma_0+\gamma_1$-a.e. Equation~\ref{eq:sup_eta_rel} then implies that $S_\e(h_1^{n_k})(\bx)$, $S_\e(h_0^{n_k})(\bx)$ converge $\gamma_1,\gamma_0$ -a.e. respectively. Because $\PP_1,\PP_0$ are marginals of $\gamma_1,\gamma_0$, this statement implies the result.
	\end{proof}

    The existence of a minimizer then follows from the fact that $S_\e(h_1^{n_k})$ converges. The next lemma describes how the $S_\e$ operation interacts with $\liminf$s and $\limsup$s.

    \begin{restatable}{lem}{lemmasupinfswap}\label{lemma:sup_liminf_swap}
    Let $h_n$ be any sequence of functions. Then the sequence $h_n$ satisfies

        \begin{equation}\label{eq:sup_liminf_swap}
        \liminf_{n\to \infty} S_\e(h_n) \geq S_\e(\liminf_{n\to \infty} h_n)  
        \end{equation}
    
and
        \begin{equation}\label{eq:sup_limsup_swap}
        \limsup_{n\to \infty} S_\e(h_n) \geq S_\e(\limsup_{n\to \infty} h_n)
        \end{equation}

\end{restatable}
See Appendix~\ref{app:S_e_liminf_limsup_swap} for a proof.

Finally, we prove that there exists a minimizer to $\Theta$ over $S_\psi$.
\begin{lemma}\label{lemma:exponential_loss_existence}
    There exists a minimizer $(h_0^*,h_1^*)$ to $\Theta$ over the set $S_\psi$.
\end{lemma}
\begin{proof}
Let $(h_0^n,h_1^n)$ be a sequence minimizing $\Theta$ over $S_\psi$.


Lemma~\ref{lemma:S_e_limit} implies that there is a subsequence $\{n_k\}$ for which $\lim_{k\to \infty} S_\e(h_0^{n_k})$ exists $\PP_0$-a.e.

Thus 
\begin{equation}\label{eq:limsup_liminf_h_0^n_k}
    \limsup_{k\to \infty} S_\e(h_0^{n_k})=\liminf_{k\to \infty} S_\e(h_0^{n_k})\quad \PP_0\text{-a.e.}
\end{equation}

Next, we will argue that the pair $(\limsup_k h_0^{n_k},\liminf_k  h_1^{n_k})$ is in $S_\psi$. Because 
\[C_\psi^*(\eta)\leq \eta h_1^{n_k}+(1-\eta) h_0^{n_k},\]
one can conclude that
\[C_\psi^*(\eta)\leq \eta \liminf_{k\to \infty}( h_1^{n_k}+(1-\eta)  h_0^{n_k})\leq \eta \liminf_{k\to \infty} h_1^{n_k}+(1-\eta) \limsup_{k\to \infty} h_0^{n_k}.\]

 Define
 \[h_1^*=\liminf_k h_1^{n_k},\quad h_0^*=\limsup_k h_0^{n_k}\]

 Now Fatou's lemma, Lemma~\ref{lemma:sup_liminf_swap}, and Equation~\ref{eq:limsup_liminf_h_0^n_k} imply that 
\begin{align}
    &\lim_{k\to \infty} \Theta(h_0^{n_k},h_1^{n_k})\geq \int \liminf_{k\to \infty}   S_\e( h_1^{n_k})d\PP_1+ \int \liminf_{k\to \infty} S_\e( h_0^{n_k})d\PP_0 &\text{(Fatou's Lemma)}\nonumber\\
    &=\int \liminf_{k\to \infty}   S_\e( h_1^{n_k})d\PP_1+ \int \limsup_{k\to \infty} S_\e( h_0^{n_k})d\PP_0 &\text{(Equation~\ref{eq:limsup_liminf_h_0^n_k})}\nonumber\\
    &\geq \int  S_\e (\liminf_{k\to \infty} h_1^{n_k})d\PP_1+ \int S_\e(  \limsup_{k\to \infty}h_0^{n_k})d\PP_0 &\text{(Lemma~\ref{lemma:sup_liminf_swap})}\nonumber\\
    &=\int  S_\e ( h_1^*)d\PP_1+ \int S_\e(  h_0^*)d\PP_0 \nonumber
\end{align}

Therefore, $(h_0^*,h_1^*)$ must be a minimizer.

\end{proof}

\section{Reducing $\Theta$ to $\prm$}\label{sec:reduction}

    Using the properties of $C_\psi^*(\eta)$, we showed in the previous section that there exist minimizers to $\Theta$ over the set $S_\psi$. The inequality $\eta h_1^*+(1-\eta^*) h_0^*\geq C_\psi^*(\eta)$ together with \eqref{eq:complimentary_slackness_restricted_necessary} imply that $h_1^*(\bx)-h_0^*(\bx)$ is a supergradient of $C_\psi^*(\eta^*(\bx))$ and thus $h_1^*-h_0^*=(C_\psi^*)'(\eta)$. This relation together with \eqref{eq:complimentary_slackness_restricted_necessary} provides two equations in two variables that can be uniquely solved for $h_0^*,h_1^*$, resulting in $h_0^*=\psi \circ -\alpha_\psi(\eta^*) ,h_1^*=\psi \circ \alpha_\psi (\eta^*)$.

    Next, primal minimizers of $\Theta$ over $S_\phi$ for \emph{any} $\phi$ will be constructed from the dual maximizers $\PP_0^*$, $\PP_1^*$ of $\bar R_\psi$. Because $\alpha_\psi(\eta)=1/2\log(\eta/1-\eta)$ is a strictly increasing function, the compositions $\psi \circ \alpha_\psi,\psi\circ- \alpha_\psi$ are strictly monotonic. 
    Thus the complimentary slackness condition \eqref{eq:sup_comp_slack_restricted} applied to $h_1^*=\psi(\alpha_\psi(\eta^*)), h_0^*=\psi(-\alpha_\psi(\eta^*))$ implies that $\supp \PP_1^*$ is contained in the set of points $\bx'$ for which $\eta^*$ assumes its infimum over some $\e$-ball at $\bx'$ and $\supp \PP_0^*$ is contained in the set of points $\bx'$ where $\eta^*$ assumes its supremum over some $\e$-ball at $\bx'$. 
    Thus, the functions $\phi \circ \alpha_\phi(\eta^*), \phi \circ -\alpha_\phi(\eta^*)$ satisfy \eqref{eq:sup_comp_slack_restricted} for the loss $\phi$. The definition of $\alpha_\phi$ further implies they satisfy \eqref{eq:complimentary_slackness_restricted_necessary}. 
    Therefore, Lemma~\ref{lemma:complimentary_slackness_restricted} implies that for \emph{any} $\phi$, $h_1^*=\phi \circ \alpha_\phi(\eta^*)$, 
    $h_0^*=\phi \circ \alpha_\phi(\eta^*)$ are primal optimal and $\PP_0^*$, $\PP_1^*$ are dual optimal!

        This reasoning about $\eta^*$ is technically wrong but correct in spirit. Because $\eta^*$ is a Raydon-Nikodym derivative, it is only defined $\PP^*$-a.e. As a result, the supremum over an $\e$-ball of the function $\phi(\alpha_\psi(\eta^*))$ is not well-defined. The solution is to replace $\eta^*$ in the discussion above by a function $\hat \eta$ that is defined everywhere. The function $\hat \eta$ is actually a version of the Raydon-Nikodym derivative $d\PP_1^*/d\PP^*$. The next two lemmas describe how one constructs this function $\hat \eta$.

    The next two lemmas discuss the analog of the $\empty^c$ transform for the Kantorovich problem in optimal transport (see for instance Chapter~1 of \citep{OptimalTransportforAppliedMathematiciansSantambroglio} or Section~2.5 of \citep{TopicsInOptimalTransportVillani}).

	\begin{lemma}\label{lemma:C_phi_transform}
		Assume that $h_0,h_1\geq 0$ and $(h_0(\bx),h_1(\bx))$ satisfy $\eta h_1+(1-\eta)h_0\geq C_\phi^*(\eta)$ for all $\eta$. 
		Then if we define $ h_0^{C_\phi^*}$ via
        \begin{equation}\label{eq:h_0_def}
            h_0^{C_\phi^*}=\sup_{\eta\in [0,1)}\frac{C_\phi^*(\eta)-\eta h_1}{1-\eta}
        \end{equation}
		then $ h_0^{C_\phi^*}\leq h_0$ while and $h_1+(1-\eta)h_0^{C_\phi^*}\geq C_\phi^*(\eta)$ for all $\eta$, and $ h_0^{C_\phi^*}$ is the smallest function $h_0$ for which $(h_0,h_1)\in S_\phi$.  Furthermore, the function $h_0^{C_\phi^*}$ is Borel and there exists a function $\bar \eta\colon \Rset^d\to [0,1]$ for which $\bar \eta(\bx)h_1(\bx)+(1-\bar \eta(\bx))h_1^{C_\phi^*}(\bx)=C_\phi^*(\bar \eta(\bx))$.
	\end{lemma}
    \begin{proof}
    For convenience, set $\td h_0=h_1^{C_\phi^*}$.
    Notice that $\td h_0\geq 0$ because the right-hand side of \eqref{eq:h_0_def} evaluates to 0 at $\eta=0$.
    We will show that $\td h_0$ is Borel and that $(\td h_0,h_1)$ is a feasible pair.

    Next, 
    Notice that the map
    \begin{equation}\label{eq:G_map}
        G(\eta,\alpha )= \begin{cases}
        \frac{ C_\phi^*(\eta)-\eta \alpha}{1-\eta}&\text{if }\eta<1\\
        \lim_{\eta\to 1} \frac{C_\phi^*(\eta)-\eta \alpha}{1-\eta}&\text{if }\eta=1
    \end{cases}
    \end{equation}
    is continuous in $\eta$. Thus, the supremum in \eqref{eq:h_0_def} can be taken over the countable set $\QQ\cap [0,1]$ and hence the function $\td h_0(\bx)=\sup_{\eta\in [0,1)\cap \QQ} G(\eta,h_1(\bx))$ is Borel measurable. Because $G(\eta,h_1(\bx))$ is continuous in $\eta$ for each fixed $\bx$, $G(\cdot, h_1(\bx))$ assumes its maximum on $\eta\in [0,1]$ for each fixed $\bx$. Thus there exists a function $\bar\eta(\bx)$ that maps $\bx$ to a maximizer of $G(\cdot, h_1(\bx))$. For this function $\bar\eta(\bx)$, one can conclude that $\td h_0(\bx)=G(\bar \eta(\bx),\bx)$ and hence
    \begin{equation}\label{eq:g_0_tilde_h_1}
        \bar \eta(\bx)h_1(\bx)+(1-\bar \eta(\bx))\td h_0(\bx)=C_\phi^*(\bar \eta(\bx)).
    \end{equation}

    Equation~\ref{eq:g_0_tilde_h_1} implies that if $f(\bx)<\td h_0(\bx)$ at any $\bx$, then $\eta h_1(\bx)+(1-\eta)f(\bx)<C_\phi^*(\eta(\bx))$ so $(f,h_1)$ is not in the feasible set $S_\phi$. Therefore, $\td h_0$ is the smallest function $f$ for which $(f,h_1)\in S_\phi$.
    \end{proof}
    
    Next we use this result to define an extension of $\eta^*$ to all of $\Rset^d$.
    \begin{lemma}\label{lemma:primal_minimizers_hat_eta}
     There exist a Borel minimizer $(h_0^*,h_1^*)$ to $\Theta$ over $S_\psi$  for which
    \begin{equation}\label{eq:hat_eta_def}
        \hat \eta(\bx)h_1^*(\bx)+(1-\hat \eta(\bx))h_0^*(\bx)=C_\psi^*(\hat \eta(\bx))
    \end{equation}
   
    for all $\bx$
    and some Borel measurable function
    $\hat \eta\colon (\supp \PP)^\e\to [0,1]$. 
\end{lemma}
    \begin{proof}
    Let $(h_0,h_1)$, be an arbitrary Borel minimizer to the primal (Lemma~\ref{lemma:exponential_loss_existence} implies that such a minimizer exists). Set $h_1^*=h_1$ and $h_0^*=h_1^{C_\psi^*}$. Then Lemma~\ref{lemma:C_phi_transform} implies that $h_0^*\leq h_0$, so $(h_0^*,h_1^*)$ is also optimal and $\eta h_1^*+(1-\eta)h_0^*\geq C_\psi^*(\eta)$ for all $\eta$. Furthermore, Lemma~\ref{lemma:C_phi_transform} implies that there is a function $\hat \eta$ for which $\hat \eta(\bx)h_1^*(\bx)+(1-\hat \eta(\bx))h_0^*(\bx)=C_\psi^*(\hat \eta(\bx))$.
    
    It remains to show that $\hat \eta$ is Borel measurable. We will express $\hat \eta(\bx)$ in terms of $h_1^*(\bx)$, and because $h_1^*(\bx)$ is Borel measurable, it will follow that $\hat \eta$ is Borel measurable as well. Because $\eta h_1^*(\bx)+(1-\eta)h_0^*(\bx)\geq C_\psi^*( \eta)$
    with equality at $\eta =\hat \eta(\bx)$, it follows that $h_1^*(\bx)-h_0^*(\bx)$ is a supergradient of $C_\psi^*$ at $\eta =\hat \eta(\bx)$. Thus Lemma~\ref{lemma:exponential_loss} implies that $h_1^*-h_0^*=(1-2\hat \eta)/\sqrt{\hat \eta(1-\hat \eta)}\Leftrightarrow h_1^*=h_0^*+(1-2\hat \eta)/\sqrt{\hat \eta(1-\hat \eta)}$. Plugging this expression and the formula $C_\psi^*(\eta)=2\sqrt{\eta(1-\eta)}$ into the relation $\hat \eta h_1^*+(1-\hat \eta) h_0^*=C_\psi^*(\hat \eta)$ results in the equation $h_0^*+\hat \eta (1-2\hat \eta)/\sqrt{\hat \eta (1-\hat \eta)}=2\sqrt{\hat \eta(1-\hat \eta)}$. Solving for $\hat \eta$ then results in $\hat \eta=(h_0^*)^2/(1+(h_0^*)^2)$. Because $h_0^*$ is Borel measurable, $\hat \eta$ is measurable as well.      
\end{proof}

    Notice that this result together with Lemma~\ref{lemma:a_n_b_n} immediately implies that $h_1^*=\psi(\alpha_\psi(\hat \eta))$ and $h_1^*=\psi(-\alpha_\psi(\hat \eta))$, immediately proving that minimizing $\Theta$ over $S_\psi$ is equivalent to minimizing $R_\psi$. Next, this observation is extended to arbitrary losses using properties of $\hat \eta$.
    Because both $\psi$ and $\alpha_\psi$ are strictly monotonic, $\hat \eta$ interacts in a particularly nice way with maximizers of the dual problem:

    	\begin{lemma}\label{lemma:gamma_supp}
			Let $\PP_0^*,\PP_1^*$ be any maximizer of $\bar R_\psi$ over $\Wball \e( \PP_0)\times  \Wball \e(\PP_1)$. Set $\PP^*=\PP_0^*+\PP_1^*$, $\eta^*=d\PP_1^*/d\PP^*$.  
			Let $\hat \eta$ be defined as in Lemma~\ref{lemma:primal_minimizers_hat_eta}.
			Then $\hat \eta=\eta^*$ $\PP^*$-a.e.
			
			Furthermore, let $\gamma_i$ be a coupling between $\PP_i,\PP_i^*$ with $\supp \gamma_i\subset \Delta_\e$. 
			Then
			\begin{equation}\label{eq:gamma_1_support}
				\supp \gamma_1\subset \{(\bx,\bx')\colon \inf_{\|\by-\bx\|\leq \e} \hat \eta(\by)=\hat \eta(\bx')\}
			\end{equation} 
				
			\begin{equation}\label{eq:gamma_0_support}
				\supp \gamma_0\subset \{(\bx,\bx')\colon \sup_{\substack{\|\by-\bx\|\leq \e}} \hat \eta(\by)=\hat \eta(\bx')\}
			\end{equation}

	\end{lemma}
        The statement $\hat \eta =\eta^*$ $\PP^*$-a.e. implies that $\hat \eta$ is in fact a version of the Raydon-Nikodym derivative $d\PP_1^*/d\PP^*$.
         
         For convenience, in this proof, we introduce the notation 
        \[I_\e(f)(\bx)=\inf_{\|\by-\bx\|\leq \e} f(\by).\]
	\begin{proof}
		Let $h_0^*,h_1^*$ be the minimizer described by Lemma~\ref{lemma:primal_minimizers_hat_eta}.  Then Lemma~\ref{lemma:a_n_b_n} implies that $h_1^*=\psi(\alpha_\psi(\hat \eta))$ and $h_0^*=\psi(-\alpha_\psi(\hat \eta))$. 
		
		The complimentary slackness condition \eqref{eq:complimentary_slackness_restricted_necessary} implies that $\eta^* h_1^*+(1-\eta^*)h_0^*=C_\psi^*(\eta^*)$ $\PP^*$-a.e., and thus Lemma~\ref{lemma:a_n_b_n} implies that $h_1^*=\psi(\alpha_\psi(\eta^*))$ and $h_0^*=\psi(\alpha_\psi(\eta^*))$ $\PP^*$-a.e. Therefore, $\psi(\alpha_\psi(\eta^*))=\psi(\alpha_\psi(\hat \eta))$ $\PP^*$-a.e. Now because the functions $\psi, \alpha_\psi$ are strictly monotonic, they are invertible. Thus it follows that $\hat \eta=\eta^*$ $\PP^*$-a.e.
		
		The complimentary slackness condition \eqref{eq:sup_comp_slack_restricted} states that 
		\[\int S_\e(h_i)(\bx)-h_i^*(\bx')d\gamma_i=0.\]
		Therefore, 
		\[S_\e(\psi(\alpha_\psi(\hat \eta)))(\bx)=\psi(\alpha_\psi(\hat \eta(\bx'))\quad \gamma_1\text{-a.e.}\quad \text{and}\quad S_\e(\psi(-\alpha_\psi(\hat \eta)))(\bx)=\psi(-\alpha_\psi(\hat \eta(\bx'))\quad \gamma_0\text{-a.e.}\]
		which implies 
		\[\psi(\alpha_\psi(I_\e(\hat \eta)(\bx)))=\psi(\alpha_\psi(\hat \eta(\bx'))\quad \gamma_1\text{-a.e.}\quad \text{and}\quad \psi(-\alpha_\psi(S_\e(\hat \eta)(\bx)))=\psi(-\alpha_\psi(\hat \eta(\bx'))\quad \gamma_0\text{-a.e.}\]
		Now $\psi,\alpha_\psi$ are both strictly monotonic and thus invertible. Therefore
		\[I_\e(\hat \eta)(\bx)=\hat \eta(\bx')\quad \gamma_1\text{-a.e.} \quad \text{and} \quad S_\e(\hat \eta)(\bx)=\hat \eta(\bx')\quad \gamma_0\text{-a.e.}\]
	\end{proof}

     Next, the relation \eqref{eq:hat_eta_def} suggests that $h_1^*=\phi \circ f^*$, $h_0^*=\phi \circ -f^*$, where $f^*$ is a function satisfying $C_\psi(\hat \eta(\bx), f^*(\bx))=C_\psi^*(\hat \eta(\bx))$. In fact, Lemma~\ref{lemma:gamma_supp} implies that this relation holds for \emph{all} loss functions, and not just the exponential loss $\psi$. To formalize this idea, we prove the following result about minimizers of $C_\psi(\eta,\cdot)$ in Appendix~\ref{app:C_phi_minimizers}:

     \begin{restatable}{lem}{lemmaCphiminimizers}\label{lemma:C_phi_minimizers}
        		Fix a loss function $\phi$ and let $\alpha_\phi(\eta)$ be as in \eqref{eq:alpha_phi_definition}. 
          Then $\alpha_\phi$ maps $\eta$ to the smallest minimizer of $C_\phi(\eta,\cdot)$. 
          Furthermore, the function $\alpha_\phi(\eta)$ non-decreasing in $\eta$.
    \end{restatable}

    Finally, we use the complimentary slackness conditions of Lemma~\ref{lemma:complimentary_slackness_restricted} to construct a minimizer $(h_0^*,h_1^*)$ to $\Theta$ over $S_\phi$ for which $h_1^*=\phi \circ f^*$, $h_0^*=\phi \circ -f^*$ for some function $f^*$.

    	\begin{lemma}\label{lemma:minimizer_construction}
			Let $\psi=e^{-\alpha}$ be the exponential loss and let $\phi$ be any arbitrary loss. Let $\PP_0^*,\PP_1^*$ be any maximizer of $\bar R_\psi$ over $\Wball \e( \PP_0)\times  \Wball \e(\PP_1)$. Define $\PP^*=\PP_0^*+\PP_1^*$ and $\eta^*=d\PP_1^*/d\PP^*$. Let $\hat \eta$ be defined as in Lemma~\ref{lemma:primal_minimizers_hat_eta}. 
			
			Then $h_0^*=\phi(-\alpha_\phi(\hat \eta))$, $h_1^*=\phi(\alpha_\phi(\hat \eta))$ minimize $\Theta$ over $S_\phi$ and $(\PP_0^*,\PP_1^*)$ maximize $\dl$ over $\Wball \e(\PP_0)\times \Wball \e (\PP_1)$. 
   
            Thus there exists a Borel minimizer to $R_\phi^\e$ and $\inf_f R_\phi^\e(f)=\inf_{(h_0,h_1)\in S_\phi} \Theta(h_0,h_1)$. 
	\end{lemma}
	\begin{proof}
		We will verify the complimentary slackness conditions of Lemma~\ref{lemma:complimentary_slackness_restricted}. 
		
		Lemma~\ref{lemma:gamma_supp} implies that $\hat \eta=\eta^*$ $\PP^*$-a.e. Therefore, $\PP^*$-a.e.,
		\[C_\phi^*(\eta^*)=C_\phi^*(\hat \eta)=\hat \eta h_1+(1-\hat \eta)h_0=\eta^* h_1+(1- \eta^*)h_0\]
		This calculation verifies the complimentary slackness condition \eqref{eq:complimentary_slackness_restricted_necessary}.

		We next verify the other complimentary slackness condition \eqref{eq:sup_comp_slack_restricted}. Let $\gamma_i$ be a coupling between $\PP_i,\PP_i^*$ with $\supp \gamma_i\subset \Delta_\e$.
		Next, because $\phi \circ \alpha_\phi$, $\phi\circ -\alpha_\phi$ are monotonic, the conditions \eqref{eq:gamma_1_support} and \eqref{eq:gamma_0_support} imply that
		\[\int \phi(\alpha_\phi(\hat \eta))d\PP_1^*=\int \phi(\alpha_\phi(\hat \eta(\bx')))d\gamma_1(\bx,\bx')=\int S_\e(\phi(\alpha_\phi(\hat \eta)))(\bx)d\gamma_1(\bx,\bx')=\int S_\e(\phi(\alpha_\phi(\hat \eta)))d\PP_1\]   
		\[\int \phi(-\alpha_\phi(\hat \eta))d\PP_0^*=\int \phi(-\alpha_\phi(\hat \eta(\bx')))d\gamma_0(\bx,\bx')=\int S_\e(\phi(-\alpha_\phi(\hat \eta)))(\bx)d\gamma_0(\bx,\bx')=\int S_\e(\phi(-\alpha_\phi(\hat \eta)))d\PP_0\]   
		This calculation verifies the complimentary slackness condition \eqref{eq:sup_comp_slack_restricted}.
	\end{proof}

    Theorems~\ref{th:strong_duality} and~\ref{th:existence_primal} immediately follow from Lemmas~\ref{lemma:duality_theta_all} and~\ref{lemma:minimizer_construction}.

    \section{Conclusion}
        We initiated the study of minimizers and minimax relations for adversarial surrogate risks. Specifically, we proved that there always exists a minimizer to the adversarial surrogate risk when perturbing in a closed $\e$-ball and a maximizer to the dual problem. Just like the results of \citep{PydiJog2021}, our minimax theorem provides an interpretation of the adversarial learning problem as a game between two players. This theory helps explain the phenomenon of transfer attacks. We hope the insights gained from this research will assist in the development of algorithms for training classifiers robust to adversarial perturbations.

\acks{Natalie Frank was supported in part by the Research Training Group in Modeling and Simulation funded by the National Science Foundation via grant RTG/DMS – 1646339. Jonathan Niles-Weed was supported in part by a Sloan Research Fellowship.}

\appendix
\section{The Universal $\sigma$-Algebra and a Generalization of Theorem~\ref{th:meas}}\label{app:meas}
\subsection{Definition of the Universal $\sigma$-Algebra and Main Measurability Result}
In this Appendix, we prove results for supremums over an arbitrary compact set, not necessarily a unit ball. For a function $g\colon \Rset^d\to \Rset^d$, we will abuse notation and denote the supremum of $g$ over the compact set $B$ by
\[S_B(g)(\bx)=\sup_{\bh\in B} g(\bx+\bh).\]


Let $X$ be a separable metric space and let $\cB(X)$ be the Borel $\sigma$-algebra on $X$. Denote the completion of $\cB(X)$ with respect to a Borel measure $\nu$ by $\cL_\nu(X)$. Let $\cM_+(X)$ be the set of all finite\footnote{Alternatively, one could compute the intersection in \eqref{eq:universal_sigma_algebra_def} over all $\sigma$-finite measures. These two approaches are equivalent because for every $\sigma$-finite measure $\lambda$ and compact set $K$, the restriction $\lambda \mres K$ is a finite measure with $\cL_{\lambda\mres K} (X)\supset \cL_\lambda(X)$.} 
positive Borel measures on $X$. Then the universal $\sigma$-algebra on $X$, $\sU(X)$ is
\begin{equation}\label{eq:universal_sigma_algebra_def}
    \sU(X)=\bigcap_{\nu\in \cM_+(X)} \cL_\nu(\Rset^d).
\end{equation}

    In other words, the universal $\sigma$-algebra is the sigma-algebra of sets which are measurable with respect to the completion of every Borel measure. Thus $\sU(X)$ is contained in $\cL_\nu(X)$ for every Borel measure $\nu$. 
        \begin{restatable}{theorem}{thuniversal}\label{th:universal}
    If $f$ is universally measurable and $B$ is a compact set, then $S_B(f)$ is universally measurable.
\end{restatable}
 Thus, if $\PP_0,\PP_1$ are Borel, integrals of the form $\int S_\e(g)d\PP_i$ in \eqref{eq:adv_phi_loss}, can be interpreted as the integral of $S_\e(g)$ with respect to the completion of $\PP_i$.

\subsection{Proof Outline}
To prove Theorem~\ref{th:universal}, we analyze the level sets of $S_B(g)$. One can compute the level set $[S_B(g)(\bx)>a]$ using a direct sum.
\begin{lemma}\label{lemma:upper_level_direct_sum}
    Let $g\colon \Rset^d\to \Rset^d$ be any function. For a set $B$, define $-B= \{-\bb \colon \bb\in B\}$. Then
    \[[S_B(g)>a]=[g>a]\oplus -B\]
\end{lemma}
\begin{proof}
To start, notice that $S_B(g)(\bx)>a$ iff there is some $\bh\in B$ for which $g(\bx+\bh)>a$. Thus 
\[\bx \in [S_B(g)>a] \Leftrightarrow\bx+\bh \in [g>a]\text{ for some }\bh \in B\Leftrightarrow \bx\in [g>a]\oplus -B \]
\end{proof}
Therefore, to show that $S_B(g)$ is measurable for measurable $g$, it suffices to show that the direct sum of a measurable set and the compact set $B$ is measurable.
Thus, to prove Theorem~\ref{th:universal}, it suffices to demonstrate the following result:

\begin{theorem}\label{th:universal_set}
    Let $A\in \sU(\Rset^d)$ and let $B$ be a compact set. Then $A\oplus B\in \sU(\Rset^d)$.
\end{theorem}

The proof of Theorem~\ref{th:universal_set} follows from fundamental concepts of measure theory. A classical measure theory result states that if $f: X\to Y$ is a continuous function, $f^{-1}$ maps Borel sets in $Y$ to Borel sets in $X$. Consider now the function $w\colon B\times \Rset^d \to B\times \Rset^d$ given by $w(\bh,\bx)=(\bh,\bx-\bh)$. Then $w$ is invertible and the inverse of $w$ is $w^{-1}(\bh, \bx+\bh)$. Furthermore, $w^{-1}$ maps the set $B\times A$ to $B\times A\oplus B$. Therefore, if $A\in \cB(\Rset^d)$, then $B\times A\oplus B$ is Borel in $\cB(B\times \Rset^d)$. However, from this statement, \emph{one cannot conclude that $A\oplus B$ is Borel in $\Rset^d$!} On the otherhand, one can use regularity of measures to conclude that $A\oplus B$ is in $\sU(\Rset^d)$. Therefore, to prove Theorem~\ref{th:universal_set}, we prove the following two results:
\begin{lemma}\label{lemma:product_univ_meas}
     Let $B\subset \Rset^d$ be a compact set. Then $B\times A\in \sU(B\times \Rset^d)$ iff $A\in \sU(\Rset^d)$.
\end{lemma}
In this document, we say a function  $f\colon X\to Y$ is \emph{universally measurable} if $f^{-1}(E)\in \sU(X)$ whenever $E\in \sU(Y)$. 
\begin{lemma}\label{lemma:Borel_to_universal_map}
    Let $f: X\to Y$ be a Borel measurable function. Then $f$ is universally measurable as well.  
\end{lemma}

This result is stated on page~171 of \citep{BertsekasShreve96}, but we include a proof below for completeness.

Lemma~\ref{lemma:Borel_to_universal_map} applied to $w$ implies that the set $B\times A\oplus B$ is universally measurable while Lemma~\ref{lemma:product_univ_meas} implies that $A\oplus B$ is universally measurable.

\subsection{Proof of Theorem~\ref{th:universal_set}}
We begin by proving Lemma~\ref{lemma:Borel_to_universal_map}.
\begin{proof}[Proof of Lemma~\ref{lemma:Borel_to_universal_map}]
    Let $A$ be a Borel set in $Y$. We will show that for any finite measure $\nu$ on $X$, $f^{-1}(A)\in \cL_\nu(X)$. As $\nu$ is arbitrary, this statement will impy that $f^{-1}(A)\in \sU(X)$.

    Consider the pushforward measure $\mu=f\sharp \nu$. This measure is a finite measure on $Y$, so by the definition of $\sU(Y)$, $A\in \cL_\mu(Y)$. Therefore, there are Borel sets $B_1\subset A\subset B_2$ in $Y$ for which $\mu(B_1)=\mu(B_2)$. Thus, $f^{-1}(B_1),f^{-1}(B_2)$ are Borel sets in $X$ for which $f^{-1}(B_1)\subset f^{-1}(A)\subset f^{-1}(B_2)$ and $\nu(f^{-1}(B_1))=\nu(f^{-1}(B_2))$. Therefore, $f^{-1}(A)\in \cL_\nu(X)$.
\end{proof}

On the other hand, the proof of Lemma~\ref{lemma:product_univ_meas} relies on the definition of a regular space $X$:
\begin{definition}
    A measure $\nu$ is \emph{inner regular} if for every Borel set $A$,
    \begin{equation*}
        \nu(A)=\sup_{\substack{K\text{ compact}\\ K\subset A}} \nu(K).    
    \end{equation*}
    
    
    The topological space $X$ is \emph{regular} if every finite Borel measure on $X$ is inner regular.
\end{definition}

The following result implies that most topological spaces encountered in applications are regular.
\begin{theorem}\label{th:regularity}
    A $\sigma$-compact locally compact Hausdorff space is regular.
\end{theorem}
This theorem is is a consequence of Theorem~7.8 of \citep{folland}.

The notion of regularity extends to complete measures.
\begin{lemma}\label{lemma:complete_regular}
    Let $\ov \nu$ be the completion of a measure $\nu$ on a regular space $X$. Then for any $A\in \cL(X)$, 
        \begin{equation*}
        \ov \nu(A)=\sup_{\substack{K\text{ compact}\\ K\subset A}} \nu(K).    
    \end{equation*}
\end{lemma}
The proof of this result is left as a exercise to the reader.

Now using the concept of regularity, we prove Lemma~\ref{lemma:product_univ_meas}.
\begin{proof}[Proof of Lemma~\ref{lemma:product_univ_meas}]
    We first prove the forward direction. Consider the projection function $\Pi_2\colon B\times \Rset^d\to \Rset^d$ given by $\Pi_2(\bx,\by)=\by$. Then $\Pi_2$ is a continuous function and $\Pi_2^{-1}(A)=B\times A$. Therefore Lemma~\ref{lemma:Borel_to_universal_map} implies that if $A$ is universally measurable in $\Rset^d$, then $B\times A$ is universally measurable in $B\times \Rset^d$.

    To prove the other direction, assume that $B\times A$ is universally measurable in $B\times \Rset^d$.
    Let $\nu$ be any finite Borel measure on $\Rset^d$. We will find Borel sets $B_1,B_2$ with $B_1\subset A\subset B_2$ for which $\nu(B_1)=\nu(B_2)$, and thus $A\in \cL_\nu(\Rset^d)$. As $\nu$ was arbitrary, it follows that $A$ is universally measurable. 
    
     Theorem~\ref{th:regularity} implies that $B\times \Rset^d$ is a regular space. Fix a Borel probability measure $\lambda$ on $B$. Then $\lambda\times \nu$ is a finite Borel measure on $B\times \Rset^d$, so it is inner regular. Let $\ov{\lambda \times \nu}$ be the completion of $\lambda \times \nu$. Then by Lemma~\ref{lemma:complete_regular},
    \[\ov{\lambda\times \nu} (B\times A)= \sup_{\substack{ K \text{ compact}\\ K\subset B\times A}} \lambda\times \nu(K) \]
    We will now argue that 
    \begin{equation}\label{eq:compact_to_product}
        \sup_{\substack{ K \text{ compact}\\ K\subset B\times A}} \lambda\times \nu(K)=\sup_{\substack{ K \text{ compact}\\ K\subset A}} \nu(K)
    \end{equation}
    Let $K\subset B\times A$ and let $\Pi_2$ be projection onto the second coordinate. Because the continuous image of a compact set is compact, K'=$\Pi_2(K)$ is compact and contained in $A$. Thus $B\times A\supset B\times K'\supset K$, which implies \eqref{eq:compact_to_product}. Now \eqref{eq:compact_to_product} applied to $A^C$ implies that 
    \[\ov{\lambda\times \nu}(X\times A)=\inf_{\substack{ U^C \text{ compact}\\ U\supset B\times A}} \lambda\times \nu(U) =\inf_{\substack{ U^C \text{ compact}\\ U\supset A}}  \nu(U).\]
    
    Thus 
    \[\sup_{\substack{ K \text{ compact}\\ K\subset  A}} \nu(K)=\inf_{\substack{ U^C \text{ compact}\\ U\supset A}}  \nu(U):=m\]
     Let $K_n$ be a sequence of compact sets contained in $A$ for which $\lim_{n\to \infty} \nu(K_n)=m$ and $U_n$ a sequence of sets containing $A$ 
     for which $U_n^C$ is compact and $\lim_{n\to \infty} \nu(U_n)=m$. Because a finite union of compact sets is compact, one can choose such sequences 
     that satisfy $K_{n+1}\supset K_n$ and $U_{n+1}\subset U_n$. Then $B_1=\bigcup K_n$, $B_2=\bigcap U_n$ are Borel sets that satisfy $B_1\subset A\subset B_2$ and $\nu(B_1)=\nu(B_2)$ so $A\in \cL_\nu(\Rset^d)$.

\end{proof}

Lastly, we formally prove Theorem~\ref{th:universal_set}.
\begin{proof}[Proof of Theorem~\ref{th:universal_set}]
	Consider the function $w\colon B\times \Rset^d \to B\times \Rset^d$ given by $w(\bh,\bx)=(\bh,\bx-\bh)$. Then $w$ is continuous, invertible, and $w^{-1}(\bh,\bx)=(\bx,\bx+\bh)$.
	
	Now let $A\in \sU(\Rset^d)$. Then Lemma~\ref{lemma:product_univ_meas} implies that $B\times \Rset^d$ is universally measurable 
 in $B\times A$. Lemma~\ref{lemma:Borel_to_universal_map} 
 then implies that $w^{-1}(B\times A)=B\times A\oplus B$ is universally measurable as well. Finally, Lemma~\ref{lemma:product_univ_meas} implies that $A\oplus B\in \sU(\Rset^d)$ as well. 
\end{proof}

\section{Alternative Characterizations of the $W_\infty$ metric}\label{app:W_infty}

    We start with proving Lemma~\ref{lemma:S_e_W_infty_equivalence} using a measurable selection theorem.
        \begin{theorem}\label{th:meas_selection}
                Let $X,Y$ be Borel sets and assume that $D\subset X\times Y$ is also Borel. Let $D_x$ denote 
        \[D_x=\{y\colon (x,y)\in D\}\] and 
        \[\proj_X(D)\colon=\{x\colon (x,y)\in D\}\]
        Let $f\colon D\to \ov \Rset $ be a Borel function mapping $D$ to $\ov \Rset$
        and define 
        \[f^*(x)=\inf_{y\in D_x} f(x,y)\]
            Assume that $f^*(x)>-\infty$ for all $x$. Then for any $\delta>0$, there is a universally measurable $\varphi\colon \proj_X(D)\to Y$ for which 
           \[f(x,\varphi(x))\leq 
                           f(x)+\delta \]
    \end{theorem}
        This statement is a consequence of Proposition~7.50 from \citep{BertsekasShreve96}.

    We use the following results about universally measurable functions, see Lemma~7.27 of \citep{BertsekasShreve96}.

        \begin{restatable}{lemma}{lemmauniversalborelaebertsekas}\label{lemma:universal_borel_a.e._bertsekas}
            Let $g:\Rset^d\to \Rset $ be a universally measurable function and let $\QQ$ be a Borel measure. Then there is a Borel measurable function $\varphi$ for which $\varphi=g$ $\QQ$-a.e.
    \end{restatable}
        This result can be extended to $\Rset^d$-valued functions:
    \begin{restatable}{lemma}{lemmauniversalborelae}\label{lemma:universal_borel_a.e.}
            Let $g:\Rset^d\to \Rset^d $ be a universally measurable function and let $\QQ$ be a Borel measure. Then there is a Borel measurable function $\varphi$ for which $\varphi=g$ $\QQ$-a.e.
    \end{restatable}
    \begin{proof}
        Let $\mathbf e_i$ denote the $i$th basis vector. Then $g_i:=\mathbf e_i\cdot g$ is a universally measurable function from $\Rset^d$ to $\Rset$, so by Lemma~\ref{lemma:universal_borel_a.e._bertsekas}, there is a Borel function $\varphi_i$ for which $\varphi_i=g_i$ $\QQ$-a.e. Then if we define $\varphi=(\varphi_1,\varphi_2,\ldots, \varphi_d)$, this function is equal to $g$ $\QQ$-a.e.
    \end{proof}

    Finally, we prove Lemma~\ref{lemma:S_e_W_infty_equivalence}. Due to Lemmas~\ref{lemma:universal_borel_a.e._bertsekas} and~\ref{lemma:universal_borel_a.e.}, this lemma heavily relies on the fact that the domain of our functions is $\Rset^d$ rather than an arbitrary metric space.
    \lemmaSeWinftyequivalence*
         Recall that this paper defines the left-left hand side of \eqref{eq:sup_integral_swap} as the integral of $S_\e(f)$ with respect to the completion of $\QQ$.
    \begin{proof}

    To start, let $\QQ'$ be a Borel measure satisfying $W_\infty(\QQ',\QQ)\leq \e$.

         Let $\gamma$ be a coupling with marginals $\QQ$ and $\QQ'$ supported on $\Delta_\e$. Then
    \begin{align*}
        &\int fd\QQ'=\int f(\bx') d\gamma(\bx,\bx')=\int f(\bx')\one_{\|\bx'-\bx\|\leq \e} d\gamma(\bx,\bx')\\
        &\leq \int S_\e(f)(\bx)\one_{\|\bx'-\bx\|\leq \e} d\gamma(\bx,\bx')=\int S_\e(f)(\bx) d\gamma(\bx,\bx')=\int S_\e(f)d\QQ    
    \end{align*}
    Therefore, we can conclude that 
    \[\sup_{\QQ'\in \Wball \e(\QQ)} \int f d\QQ'\leq \int S_\e(f)d\QQ.\]

    We will show the opposite inequality by applying the measurable selection theorem. 
    Theorem~\ref{th:meas_selection} implies for each $\delta>0$, one can find a universally measurable function $\varphi \colon \Rset^d\to \ov{B_\e(\bx)}$ for which $f(\varphi(\bx))+\delta\geq S_\e(f)(\bx)$. 
    By Lemma~\ref{lemma:universal_borel_a.e.}, one can find a Borel measurable function $T$ for which $T=\varphi$ $\QQ$-a.e.

    Let $\QQ'=\QQ\circ T^{-1}$. Because $T$ is Borel measurable, $\QQ'$ and $f\circ T$ are Borel. 
    We will now argue that $\int fd\QQ'+\delta\geq \int S_\e(f)d\QQ$. Recall that $\varphi$ is always measurable with respect to the completion of $\QQ$, and by convention $\int gd\QQ$ means integration with respect to the completion of $\QQ$.
    Then if we define $M=\QQ(\Rset^d)$,
    \[\int f d\QQ'=\int fd \QQ\circ T^{-1}=\int f(T(\bx))d\QQ=\int f(\varphi(\bx))d \QQ \geq \int S_\e(f)-\delta d\QQ= \int S_\e(f) d\QQ -\delta M \]
    Because $\delta>0$ was arbitrary and $\QQ'\in \Wball \e (\QQ)$,
    \[\int S_\e(f)d\QQ\leq \sup_{\QQ'\in \Wball \e (\QQ)} \int fd\QQ'\]

    It remains to show that $W_\infty(\QQ,\QQ')\leq \e$. Define a function $G\colon \Rset^d\to \Rset^d\times \Rset^d$, $G(\bx)=(\bx,T(\bx))$ and a coupling $\gamma$ by $\gamma=G\sharp \QQ$. Then $\gamma(\Delta_\e) = G\sharp(\QQ)(\Delta_\e) = \QQ(G^{-1}(\Delta_\e)) = 1$, so $\supp(\gamma) \subseteq \Delta_\e$.
\end{proof}

Next we prove Lemma~\ref{lemma:W_inf_integral_characterization}. We begin by presenting Strassen's theorem, see Corollary~1.28 of \citep{TopicsInOptimalTransportVillani} for more details
\begin{theorem}[Strassen's Theorem]\label{th:strassen}
Let $\PP,\QQ$ be positive finite measures with the same mass and let $\e\geq 0$. Let $\Pi(\PP,\QQ)$ denote the set couplings of $\PP$ and $\QQ$. Then
\begin{equation}\label{eq:strassen_eq}
    \inf_{\pi\in \Pi(\PP,\QQ)} \pi(\{\|\bx-\by\|>\e)=\sup_{A\text{ closed}} \QQ(A)-\PP(A^\e)
\end{equation}
\end{theorem}
Strassen's theorem is usually written with $A^\e$ in \eqref{eq:strassen_eq} replaced by $A^{\e]}= \{\bx\colon \dist(\bx,A)\leq \e\}$-- however, for closed sets $A^{\e]}=A^\e$.
Strassen's theorem together with Urysohn's lemma then immediately proves Lemma~\ref{lemma:W_inf_integral_characterization}.
\begin{lemma}[Urysohn's Lemma]\label{lemma:urysohn}
    Let $A$ and $B$ be two closed and disjoint subsets of $\Rset^d$. Then there exists a function $f\colon \Rset^d\to [0,1]$ for which $f=0$ on $A$ and $f=1$ on $B$.
\end{lemma}
See for instance result~4.15 of \citep{folland}.

\lemmaWinfintegralcharacterization*
\begin{proof}
    First, notice that Lemma~\ref{lemma:S_e_W_infty_equivalence} implies that if $\QQ\in \Wball \e(\PP)$, then $\int S_\e(h)d\PP\geq \int hd\QQ$ for all $h\in C_b(\Rset^d)$, proving the inequality $\geq$ in the statement of the lemma.
    
    We will now argue the other inequality: specifically, we will show that 
    \begin{equation}\label{eq:S_ef_A^e}
        \sup_{A\text{ closed}} \QQ(A)-\PP(A^\e) \leq \sup_{h\in C_b(\Rset^d)} \int h d\QQ- \int S_\e(h)d\PP
    \end{equation}
    Strassen's theorem will then imply that $W_\infty(\PP,\QQ)\leq \e$. Let $\delta$ be arbitrary and let $A$ be a closed set that satisfies $\sup_{A\text{ closed}} \QQ(A)-\PP(A^\e)\leq \QQ(A)-\PP(A^\e)+\delta$. Now because $A$ is closed, $A_n=A\oplus B_{1/n}(\zero)$ is a series of open sets decreasing to $A$ and $A_n^\e=A^\e\oplus B_{1/n}(\zero)$ is a sequence of open sets decreasing to $A^\e$. Thus pick $n$ sufficiently large so that $\PP(A^{\e}_n-\PP(A^\e)\leq \delta$. By Urysohn's lemma, one can choose a function $h$ which is 1 on $A$, 0 on $ A_n^C$, and between 0 and 1 on $A_n-A^C$. Then $S_\e(h)$ is $1$ on $A^\e$, 0 on $(A^{\e}_n)^C$ and between 0 and 1 on $A^{\e}_n-A^\e$. Then $\int hd\QQ -\QQ(A)\geq 0$ and thus
    \[\left(\int h d\QQ-\int S_\e(h)d\PP\right) -\left( \QQ(A)-\PP(A^\e)\right)\geq \PP(A^\e)-\PP(A_n^\e)\geq -\delta.\]
    Because $\delta$ was arbitrary, \eqref{eq:S_ef_A^e} follows.
\end{proof}

\section{Minimizers of $C_\phi(\eta,\cdot)$: Proof of Lemma~\ref{lemma:C_phi_minimizers}}\label{app:C_phi_minimizers}

\lemmaCphiminimizers*
\begin{proof}
     To start, we will show that $\alpha_\phi(\eta)$ as defined in \eqref{eq:alpha_phi_definition} is a minimizer of $C_\phi(\eta,\cdot)$.
     Let $S$ be the set of minimizers of $C_\phi^*(\eta,\cdot)$, which is non-empty due to the lower semi-continuity of $\phi$. 
     Let $a=\inf S=\alpha_\phi(\eta)$ and
     let $s_i\in S$ be a sequence converging to $a$. Then because $\phi$ is lower semi-continuous,
    \[C_\phi^*(\eta)=\liminf_{i\to \infty} \eta\phi(s_i)+(1-\eta)\phi(-s_i)\geq \eta\phi(a)+(1-\eta)\phi(-a)\] Then $a$ is in fact a minimizer of $C_\phi^*(\eta,\cdot)$, so it is the smallest minimizer of $C_\phi^*(\eta,\cdot)$.

     We will now show that the function $\alpha_\phi$ is non-decreasing. 

    One can write 
    
     \begin{align}
        C_\phi(\eta_2,\alpha)&=\eta_2\phi(\alpha)+(1-\eta_2)\phi(-\alpha)\nonumber \\
        &=\eta_1\phi(\alpha)+(1-\eta_1)\phi(-\alpha)+(\eta_2-\eta_1) (\phi(\alpha)-\phi(-\alpha))\nonumber \\
        &=C_\phi(\eta_1,\alpha)+(\eta_2-\eta_1)(\phi(\alpha)-\phi(-\alpha)) \label{eq:eta_2_rep}
        \end{align}
    
    Notice that the function $\alpha\mapsto \phi(\alpha)-\phi(-\alpha)$ is non-increasing. 
    Then because $\alpha_\phi(\eta_1)$ is the smallest minimizer of $C_\phi(\eta_1,\alpha)$, if $\alpha<\alpha_\phi(\eta_1)$, then $C_\phi(\eta_1,\alpha)>C_\phi(\eta_1,\alpha_\phi(\eta_1))$. 
    Furthermore, $\phi(\alpha)-\phi(-\alpha)\geq \phi(\alpha_\phi(\eta_1))-\phi(-\alpha_\phi(\eta_1))$. 
    Therefore, \eqref{eq:eta_2_rep} implies that $C_\phi(\eta_2,\alpha)>C_\phi(\eta_2,\alpha_\phi(\eta_1))$, and thus $\alpha$ cannot be a minimizer of $C_\phi(\eta_2,\cdot)$. Therefore, $\alpha_\phi(\eta_2)\geq \alpha_\phi(\eta_1)$.
    
\end{proof}
\section{Continuity properties of $\dl$-- Proof of Lemma~\ref{lemma:Xi_star_pre}}\label{app:dl_C_b}

	 	Recall the function $G(\eta,\alpha)$ defined by \eqref{eq:G_map}. With this notation, one can write the $C_\phi^*$ transform as $h_1^{C_\phi^*}=\sup_{\eta\in [0,1]} G(\eta, h_1)$.
	\begin{lemma}\label{lemma:C_phi_transform_bounded}
	Let $c>0$ and consider $\alpha\geq c$. Let $a(\alpha)=\alpha^{C_\phi^*}$, where the $C_\phi^*$ transform is as in Lemma~\ref{lemma:C_phi_transform}. Then there is a constant $k<1$ for which
	\begin{equation}\label{eq:a_alpha}
		a(\alpha)=\sup_{\eta\in [0,k]}\frac{C_\phi^*(\eta)-\eta \alpha}{1-\eta}
	\end{equation}
			The constants $k$ depends only on $c$. 
	\end{lemma}
	\begin{proof}
		Recall that the function $G(\eta,\alpha)$
			is decreasing in $\alpha$ for fixed $\eta$ and continuous on $[1,0)$.
			Let $k=\sup \{\eta\colon G(\eta,c)>0\}$. As $c$ is strictly positive, one can conclude that $\lim_{\eta\to 1}G(\eta,c)=-\infty$ and as a result $k<1$.  Because $G$ is decreasing in $\alpha$, one can conclude that $G(\eta, \alpha)\leq 0$ for all $\eta>k$ and $\alpha\geq c$. However, $\sup_{\eta\in [0,1]}G(\eta,\alpha)\geq 0$ because $G(0,\alpha)=0$ for all $\alpha$. Thus \eqref{eq:a_alpha} holds.
			
	\end{proof}
	
	\begin{lemma}\label{lemma:sup_lip}
		Let $\{f_\alpha\}$ be a set of $L$-Lipschitz functions. Then $\sup_\alpha f_\alpha$ is also $L$-Lipschitz.
	\end{lemma}
		This statement is proved in Box~1.8 of \citep{OptimalTransportforAppliedMathematiciansSantambroglio}.
	
	\begin{lemma}\label{lemma:continuous_approximation}
		Let $\QQ$ be any finite measure and assume that $g$ is a non-negative function in $L^1(\QQ)$. Let $\delta>0$. Then there is a lower semi-continuous function $\tilde g$ for which $\int |g-\tilde g|<\delta$ and $g\geq 0$.
		
	\end{lemma}
	See Proposition~7.14 of Folland.
	
	\begin{lemma}\label{lemma:phi_lip_approximation}
		Let $g$ be a lower semi-continuous function bounded from below. Then there is a sequence of Lipschitz functions that approaches $g$ from below.
		
	\end{lemma}
	This statement appears in Box~1.5 of \citep{OptimalTransportforAppliedMathematiciansSantambroglio}.
	\begin{corollary}\label{cor:lip_approximation}
		Let $h$ be an $L^1(\QQ)$ function with $h\geq 0$. Then for any $\delta$, there exists a Lipschitz $\tilde h$ for which $\int |h-\tilde h|d\QQ<\delta$. 
	\end{corollary}
	\begin{proof}
		By Lemma~\ref{lemma:continuous_approximation}, one can pick a lower semi-continuous $\tilde g$ for which $\td g\geq 0$ and $\int |h-\td g|d\QQ<\delta/2$. Next, by Lemma~\ref{lemma:phi_lip_approximation}, one can pick a Lipschitz $\td h$ for which $\int |\td g-\td h|d\QQ\leq \delta/2$. Thus $\int |h-\td h|d\QQ<\delta$.
	\end{proof} 
    \lemmaxistarpre*
	\begin{proof}
		Let $\PP'=\PP_0'+\PP_1'$ and $\eta'=d\PP_1'/d\PP'$. Then for any $(h_0,h_1)\in S_\phi \cap E$,
		
		\[\int h_1d\PP_1'+\int h_0d\PP_0'=\int \eta'h_1+(1-\eta')h_0 d\PP'\geq \int C_\phi^*(\eta')d\PP'=\dl(\PP_0',\PP_1').\]
		We will now focus on showing the other inequality.
		Define a function $f$ by 
		\[f(\bx)=
		\begin{cases}
			\alpha_\phi(\eta'(\bx))&\bx\in \supp \PP'\\
			0                      &\bx\not \in \supp \PP' 	
		\end{cases}\]
		Let $h_1=\phi \circ f$, $h_0=\phi \circ -f$. Then $h_1$, $h_0$ satisfy the inequality $\eta h_1+(1-\eta)h_0\geq C_\phi^*(\eta)$ for all $\eta$ while on $\supp \PP'$,  $\eta'(\bx)h_1(\bx)+(1-\eta'(\bx))h_0(\bx)=C_\phi^*(\eta')$ and therefore 
		\[\int h_1d\PP_1'+\int h_0d\PP_0'=\int \eta' h_1+(1-\eta')h_0d\PP'=\int C_\phi^*(\eta')d\PP'.\]
		However, $(h_0,h_1)\not \in E$. We will now approximate $h_0,h_1$ by bounded continuous functions contained in $S_\phi$. Let $\delta>0$ be arbitrary. Pick a constant $c>0$ for which $\int cd\PP'<\delta$ and set $\td h_1=\max(h_1,c)$. 
		The pair $(h_0,\td h_1)$ are feasible pair for the set $S_\phi$,
		and thus
		\begin{equation}\label{eq:feasible_ineq}
			C_\phi^*(\eta)-\eta \td h_1-(1-\eta)h_0\leq 0
		\end{equation}
		 Furthermore, 
		\begin{equation}\label{eq:td_h_0_1}
			\int \td h_1 d\PP_1'+\int  h_0d\PP_0'<\dl(\PP_0',\PP_1')+\delta.	
		\end{equation}

		Let $k$ be the constant described by Lemma~\ref{lemma:C_phi_transform_bounded} corresponding to $c$.
		Now by Corollary~\ref{cor:lip_approximation}, there is a Lipschitz function $g$ for which $\int|h_1-g|d\PP'<\min((1-k)/k,1)\delta$. Let $\hat h_1=\max(g,c)$. Then Lemma~\ref{lemma:sup_lip} implies that $\hat h_1$ has the same Lipschitz constant as $g$, and the fact that $\tilde h_1\geq c$ implies that 
		\begin{equation}\label{eq:bar_h_1}
			\int |\tilde h_1-\hat h_1|d\PP'\leq \int |\tilde h_1-g|d\PP'<\min\left(\frac  {1-k}{k},1\right) \delta
		\end{equation}
		Now let $\hat h_0=\hat h_1^{C_\phi^*}$. By Lemma~\ref{lemma:C_phi_transform_bounded}, the supremum in the $\empty^{C_\phi^*}$ transform for computing $\hat h_0$ can be taken over $[0,k]$. Therefore, if $L$ is the Lipschitz constant of $\hat h_1$, Lemma~\ref{lemma:sup_lip} implies that the Lipschitz constant of $\hat h_0$ is at most $kL/(1-k)$. Furthermore, $\hat h_0,\hat h_1$ are bounded on $K^\e$ because Lipschitz functions are bounded over compact sets. Thus $(\hat h_0,\hat h_1)$ is in $S_\phi \cap E$.
		Next, we will show that $\int \hat h_0$ is close to $\int h_0$.
		\begin{align*}
			&\int \hat h_0-h_0d\PP_0'=\int \sup_{[0,k]} \frac{C_\phi^*(\eta)-\eta \hat h_1}{1-\eta} -h_0 d\PP_0'=\int \sup_{[0,k]}\frac{C_\phi^*(\eta)-\eta \hat h_1-(1-\eta)h_0}{1-\eta}d\PP_0'\\
			&= \int  \sup_{[0,k]}\left(\frac{C_\phi^*(\eta)-\eta \tilde h_1-(1-\eta)h_0}{1-\eta}+ \frac \eta {1-\eta} (\td h_1-\hat h_1)\right) d\PP_0'\\
			&\leq\int  \sup_{[0,k]}\frac{C_\phi^*(\eta)-\eta \tilde h_1-(1-\eta)h_0}{1-\eta}+ \sup_{[0,k]}\frac \eta {1-\eta} (\td h_1-\hat h_1)d\PP_0' \leq \int \sup_{[0,k]} \frac \eta {1-\eta} (\tilde h_1 -\hat h_1)d\PP_0'&\text{(Equation~\ref{eq:feasible_ineq})}\\
			&=\frac {k}{1-k} \int \tilde h_1-\hat h_1d\PP_0'\leq \delta&\text{(Equation~\ref{eq:bar_h_1})}
		\end{align*}

		Therefore, by \eqref{eq:td_h_0_1}, \eqref{eq:bar_h_1}, and the computation above, 
		\[\int \hat h_1 d\PP_1'+\int \hat h_0d\PP_0' \leq \dl(\PP_0',\PP_1') +3\delta\]
       AS $\delta>0$ is arbitrary, this inequality implies \eqref{eq:dl_inf_representation}. Because $K^\e$ is compact, the upper semi-continuity and concavity of $\dl$ then follows from \eqref{eq:dl_inf_representation} together with the Reisz representation theorem.
	\end{proof}

\section{Duality for Distributions with arbitrary support--Proof of Lemma~\ref{lemma:duality_theta_all}}\label{app:duality_all}
We begin with the simple observation that weak duality holds for measures supported on $\Rset^d$. This argument is essentially swapping the order of an infimum and a supremum as presented in Section~\ref{sec:main_results_summary}.

\begin{lemma}[Weak Duality]\label{lemma:weak_duality_borel}
    Let $\phi$ be a non-increasing and lower semi-continuous loss function. Let $S_\phi$ be the set of pairs of functions defined in \eqref{eq:S_def} for $K=\Rset^d$.

     Then 
     \begin{equation*}
         \inf_{(h_0,h_1)\in S_\phi}\Theta(h_0,h_1)\geq \sup_{\substack{\PP_0'\in \Wball\e (\PP_0)\\ \PP_1'\in \Wball\e(\PP_1)}} \bar R_\phi(\PP_0',\PP_1')
     \end{equation*}
\end{lemma}
\begin{proof}
    By Lemma~\ref{lemma:S_e_W_infty_equivalence}, 
    \[\inf_{(h_0,h_1)\in S_\phi}\int S_\e(h_0)d\PP_0+\int S_\e(h_1)d\PP_1=\inf_{(h_0,h_1)\in S_\phi}\sup_{\substack{\PP_0'\in \Wball\e (\PP_0)\\ \PP_1'\in \Wball\e(\PP_1)}} \int h_0d\PP_0'+\int h_1d\PP_1'.\]
    Thus by swapping the $\inf$ and the $\sup$,
    \begin{align*}
        &\inf_{(h_0,h_1)\in S_\phi} \int S_\e(h_0)d\PP_0+\int S_\e(h_1)d\PP_1 \geq \sup_{\substack{\PP_0'\in \Wball\e (\PP_0)\\ \PP_1'\in \Wball\e(\PP_1)}} \inf_{(h_0,h_1)\in S_\phi}\int h_0d\PP_0'+\int h_1d\PP_1'\\
        &=\sup_{\substack{\PP_0'\in \Wball\e (\PP_0)\\ \PP_1'\in \Wball\e(\PP_1)}}\inf_{(h_0,h_1)\in S_\phi}\int \frac{d\PP_1'}{d(\PP_0'+\PP_1')} h_1+\left( 1-\frac{d\PP_1'}{d(\PP_0'+\PP_1')} \right) h_0 d(\PP_0'+\PP_1')
        \geq\sup_{\substack{\PP_0'\in \Wball\e (\PP_0)\\ \PP_1'\in \Wball\e(\PP_1)}}\bar R_\phi(\PP_0',\PP_1')
    \end{align*}
    
\end{proof}

The main strategy in this section is approximating measures with unbounded support by measures with bounded support. To this end, we define the \emph{restriction} of a measure $\PP$ to a set $K$ by $\PP|_K(A)=\PP(K\cap A)$. 

The Portmaneau theorem then allows us to draw some conclusions about weakly convergent sequences of measures.
\begin{theorem}[Portmanteau Theorem]\label{th:portmanteau}
The following are equivalent:
\begin{enumerate}[label=\arabic*)]
\item The sequence $\QQ^n\in \cM_+(\Rset^d)$ converges weakly to $\QQ$
\item \label{it:portmanteau_closed}For all closed sets $C$, $\limsup_{n\to \infty} \QQ^n(C)\leq \QQ(C)$ and $\lim_{n\to \infty} \QQ^n(\Rset^d)=\QQ(\Rset^d)$
\item For all open sets $U$,
$\liminf_{n\to \infty} \QQ^n(U)\geq \QQ(U)$ and $\lim_{n\to \infty}\QQ^n(\Rset^d)=\QQ(\Rset^d)$
\end{enumerate}
\end{theorem}
See Theorem 8.2.3 of \citep{BogachevMeasureTheory2}. This result allows us to draw conclusions about restrictions of weakly convergent sequences.
\begin{lemma}\label{lemma:weakly_convergent_restriction}    
    Let $\QQ^n, \QQ\in \cM_+(\Rset^d)$ and assume that $\QQ^n$ converges weakly to $\QQ$. Let $K$ be a compact set with  $\QQ(\partial K)=0$. Then $\QQ^n|_K$ converges weakly to $\QQ|_K$.
\end{lemma}
\begin{proof}
    We will verify \ref{it:portmanteau_closed} of Theorem~\ref{th:portmanteau} for the measures $\QQ^n|_K$, $\QQ$.
    
   First, because $\QQ(K)=\QQ(\interior K)$, Theorem~\ref{th:portmanteau} 
   implies that \[\limsup_{n\to \infty} \QQ^n(K)\leq \QQ(K)=\QQ(\interior K)\leq \liminf_{n\to \infty} \QQ^n(\interior K)\leq \liminf_{n\to \infty} \QQ^n(K).\]
  Therefore, $\lim_{n\to \infty} \QQ^n|_K(\Rset^d)=\lim_{n\to \infty} \QQ^n(K)=\QQ(K)$. Next, for any closed set $C$, the set $C\cap K$ is also closed so the fact that $\QQ^n$ weakly converges to $\QQ$ implies that 
\[\limsup_{n\to \infty} \QQ^n|_K(C)=\limsup_{n\to \infty} \QQ^n(K\cap C)\leq \QQ(K\cap C)=\QQ|_K(C).\]  
\end{proof}
Next, Prokhorov's theorem allows us to identify weakly convergent subsequences.
\begin{theorem}\label{th:prokhorov}
    Let $\QQ^n$ be a sequence of measures for which $\sup_n\QQ^n(\Rset^d)<\infty$ and for all $\delta$, there exists a compact $K$ for which $\QQ^n(K^C)<\delta$ for all $n$. Then $\QQ^n$ has a weakly convergent subsequence.
\end{theorem}
See Theorem 8.6.2 of \citep{BogachevMeasureTheory2}. These results imply that $\dl$ is upper semi-continuous on $\cM_+(\Rset^d)\times \cM_+(\Rset^d)$.

\begin{lemma}\label{lemma:dl_upper_semicontinuous}
    The functional $\dl$ is upper semi-continuous with respect to the weak topology on probability measures (in duality with $C_0(\Rset^d)$).
\end{lemma}
Notice that Lemma~\ref{lemma:Xi_star_pre} implies that $\dl$ is upper semi-continuous on the space $\cM_+(K^\e)\times \cM_+(K^\e)$ for a \emph{compact} set $K$. However, on $\Rset^d$, weak convergence of measures is defined with respect to the dual of $C_0(\Rset^d)$, the set of continuous functions vanishing at $\infty$. This set is strictly smaller than $C_b(\Rset^d)$, and thus the relation \eqref{eq:dl_inf_representation} would not immediately imply the the upper semi-continuity of $\prm$.
\begin{proof}
    Let $\QQ_0^n,\QQ_1^n$ be sequences of measures converging to $\QQ_0,\QQ_1$ respectively. Set $\QQ=\QQ_0+\QQ_1$.
       
    Define a function $F(R)= \QQ(\ov{B_R(\zero)}^C)$. Then because this function is non-increasing, it has finitely many points of discontinuity. 
    
    Let $\delta>0$ be arbitrary and choose $R$ large enough so that $F(R)<\delta/C_\phi^*(1/2)$ and $F$ is continuous at $R$.
    Then notice that $\PP(\partial B_R(\zero))=0$ and thus one can apply Lemma~\ref{lemma:weakly_convergent_restriction} with the set $\ov{B_R(\zero)}$.

Now let $\nu_0,\nu_1$ be arbitrary measures. Consider $\nu_i^R$ defined by $\nu_i^R=\nu_i|_{\ov{B_R(\zero)}}$.    
    Set $\nu=\nu_0+\nu_1$, $\eta=d\nu_1/d\nu$, $\nu^R=\nu_0^R+\nu_1^R$, $\eta^R=d\nu_1^R/d\nu^R$. Then on $\ov{B_R(\zero)}$, $\eta^R=\eta$ a.e. Thus
    \begin{equation}
    |\dl(\nu_0^R,\nu_1^R)-\dl(\nu_0,\nu_1)|=\left| \int C_\phi^*(\eta)\one_{\ov{B_R(\zero)}}d\nu-\int C_\phi^*(\eta)d\nu \right|\leq C_\phi^*\left( \frac 12 \right)\nu(\ov{B_R(\zero)^C})\label{eq:bound_last}
    \end{equation} 
    If we define $\QQ_{i,R},\QQ_{i,R}^n$ via $\QQ_{i,R}=\QQ_i|_{\ov{B_R(\zero)}}$,  $\QQ_{i,R}^n=\QQ_{i}^n=\QQ_i^n|_{\ov{B_R(\zero)}}$, Lemma~\ref{lemma:weakly_convergent_restriction} implies that $\QQ_{i,R}^n$ converges weakly to $\QQ_{i,R}$ and $\lim_{n\to \infty} \QQ^n(\ov{B_R(\zero)}^C)=\QQ(B_R(\zero)^C)<\delta$. Therefore, for sufficiently large $n$, $\QQ^n(\ov{B_R(\zero)}^C)<2\delta/C_\phi^*(1/2)$. By Lemma~\ref{lemma:Xi_star_pre} and \eqref{eq:bound_last},
    \[\limsup_{n\to \infty} \dl(\QQ_0^n,\QQ_1^n)\leq \limsup_{n\to \infty} \dl(\QQ_{0,R}^n,\QQ_{1,R}^n)+2\delta\leq \dl(\QQ_{0,R},\QQ_{1,R})+2\delta\leq \dl(\QQ_0,\QQ_1)+3\delta\]

    Because $\delta$ was arbitrary, the result follows.
    
\end{proof}

Next we consider an approximation of $\PP_0$, $\PP_1$ by compactly supported measures.

\begin{lemma}\label{lemma:dual_first_approximation}
Let $\PP_0,\PP_1$ be finite measures. Define $\PP_i^n=\PP_i|_{\ov{B_n(\zero)}}$ for $n\in \Nset$. Then $\PP_0^n,\PP_1^n$ converge weakly to $\PP_0$, $\PP_1$ respectively. Furthermore, there are measures $\PP_0^*\in \Wball \e (\PP_0),\PP_1^*\in \Wball \e(\PP_1)$ for which 
\begin{equation}\label{eq:dl_upper_approximation_bound}
    \limsup_{n\to \infty} \sup_{\substack{\PP_1'\in \Wball \e(\PP_1^n)\\ \PP_0' \in \Wball \e(\PP_0)^n}}\dl (\PP_0',\PP_1')\leq \dl (\PP_0^*,\PP_1^*)
    \end{equation}
\end{lemma}
\begin{proof}
    Set $\PP=\PP_0+\PP_1$, $\PP^n=\PP_0^n+\PP_1^n$. 
    Notice that \ref{it:portmanteau_closed} of Theorem~\ref{th:portmanteau} 
    implies that $\PP_i^n$ converges weakly 
    to $\PP_i$. Let $\PP_0^{*,n},\PP_1^{*,n}$ be 
    maximizers of $\dl$ over $\Wball \e(\PP_0^n)\times \Wball \e(\PP_1^n)$. 
    Next, by Strassen's theorem (Theorem~\ref{th:strassen}),  $\PP_i^{n}(\ov{B_r(\zero)})\leq \PP_i^{n,*}(\ov{B_{r+\e}(\zero)})$ and 
    thus $\PP_i(\ov{B_r(\zero)}^C)\geq \PP_i^n(\ov{B_r(\zero)}^C)\geq \PP_i^{n,*}(\ov{B_{r+\e}(\zero)})$. Therefore, one can apply Prokhorov's theorem (Thereom~\ref{th:prokhorov}) to conclude that $\PP_0^{n,*}$, $\PP_1^{n,*}$ have subsequences $\PP_0^{n_k,*}$, $\PP_1^{n_k,*}$ that converge to measures $\PP_0^*,\PP_1^*$ respectively.  The upper semi-continuity of $R_\phi$ (Lemma~\ref{lemma:dl_upper_semicontinuous}) then implies that $\PP_0^*,\PP_1^*$ satisfy \eqref{eq:dl_upper_approximation_bound}.

    It remains to show that $\PP_i^*\in \Wball \e (\PP_i)$. We will apply Lemma~\ref{lemma:W_inf_integral_characterization}. Because $\PP_i^{n_k,*}\in \Wball \e(\PP_i^{n_k})$ for all $n_k$, Lemma~\ref{lemma:W_inf_integral_characterization} implies that for every $f\in C_b(\Rset^d)$,
    $\int S_\e(f)d\PP_i^{n_k}\geq \int fd\PP_i^{*,n_k}$.
    Because $\PP_i^{n_k}$ converges weakly to $\PP_i$ and $\PP_i^{*,n_k}$ converges weakly to $\PP_i^*$, one can take the limit $k\to \infty$ to conclude
    $\int S_\e(f)d\PP_i\geq \int fd\PP_i^*$ for all $f\in C_b(\Rset^d)$. Lemma~\ref{lemma:W_inf_integral_characterization} then implies $\PP_i^*\in \Wball \e(\PP_i)$.

\end{proof}

\lemmadualitythetaall*
\begin{proof}
    Let $\PP_0^n$, $\PP_1^n$, $\PP_0^*,\PP_1^*$ be the the measures described in Lemma~\ref{lemma:dual_first_approximation}. Notice that because $\PP_0^n$, $\PP_1^n$ are compactly supported, Lemma~\ref{lemma:duality_compact} applies. Define
    \[\Theta^n(h_0,h_1)=\int S_\e(h_1)d\PP_1^n+\int S_\e(h_0)d\PP_0^n.\]

    Thus Lemmas~\ref{lemma:duality_compact} and Lemma~\ref{lemma:dual_first_approximation} imply that  
    \begin{equation}\label{eq:right_dl_approximation}
        \limsup_{n\to \infty}\inf_{(h_0,h_1)\in S_\phi} \Theta^n(h_0,h_1)=\limsup_{n\to \infty}\sup_{\substack{\PP_0'\in \Wball\e (\PP_0^n)\\ \PP_1'\in \Wball\e(\PP_1^n)}} \bar R_\phi(\PP_0',\PP_1')\leq \dl(\PP_0^*,\PP_1^*)\leq \sup_{\substack{\PP_0'\in \Wball \e (\PP_0)\\ \PP_1'\in \Wball \e(\PP_1)}}\dl(\PP_0',\PP_1').
    \end{equation}

We will show 
    \begin{equation}\label{eq:compact_approximation}
        \inf_{(h_0,h_1)\in S_\phi} \Theta(h_0,h_1) \leq  \limsup_{n\to \infty} \inf_{(h_0,h_1)\in S_\phi}  \Theta_n(h_0,h_1).   
    \end{equation}    
    Equations~\ref{eq:right_dl_approximation} and~\ref{eq:compact_approximation}
    imply that
    \begin{equation}\label{eq:opposite_weak_duality}\inf_{(h_0,h_1)\in S_\phi} \Theta(h_0,h_1)\leq \dl(\PP_0^*,\PP_1^*)\leq \sup_{\substack{\PP_0'\in \Wball\e (\PP_0)\\ \PP_1'\in \Wball\e(\PP_1)}} \bar R_\phi(\PP_0',\PP_1').
    \end{equation}

    This relation together with weak duality (Lemma~\ref{lemma:weak_duality_borel}) imply that the inequalities in \eqref{eq:opposite_weak_duality} are actually equalities. Therefore strong duality holds and $\PP_0^*,\PP_1^*$ maximizes the dual.
    
    Next, we prove the inequality in  \eqref{eq:compact_approximation}.
    Let $\delta>0$ be arbitrary and choose an $n\in \Nset$ for which $n>2\e$ and 
    \begin{equation}\label{eq:delta_choice}
        \PP_1(\ov{B_{n-2\e}(\zero)}^C)+\PP_0(\ov{B_{n-2\e}(\zero)}^C)\leq \delta
    \end{equation}
    Let $(h_0^n,h_1^n)\in S_\phi$ be functions for which 
    \begin{equation}\label{eq:h_0^n_h_1^n}
        \Theta^n(h_0^n,h_1^n)\leq \inf_{(h_0,h_1)\in S_\phi}\Theta^n(h_0,h_1)+\delta
    \end{equation}
    Define
    \[\tilde h_0^n= \begin{cases}
        h_0^n &\bx \in \ov{B_{n-\e}(\zero)}\\
        C_\phi^*\left(\frac 12 \right) &\bx \not \in \ov{B_{n-\e}(\zero)}
    \end{cases}\quad
\tilde h_1^n= \begin{cases}
        h_1^n &\bx \in \ov{B_{n-\e}(\zero)}\\
        C_\phi^*\left(\frac 12 \right) &\bx \not \in \ov{B_{n-\e}(\zero)}
    \end{cases}\]
    Because $\eta h_0^n+(1-\eta)h_1^n\geq C_\phi^*(\eta)$  $\forall \eta\in[0,1]$ on $B_{n-\e}(\zero)$ and $(C_\phi^*(1/2),C_\phi^*(1/2))\in S_\phi$, one can conclude that $(\tilde h_0^n,\tilde h_1^n)\in S_\phi$.

    Now because $n>2\e$, the regions $\ov{B_{n-\e}(\zero)}, \ov{B_{n-2\e}(\zero)}$ are non-empty. One can bound $S_\e(\tilde h_i)$ in terms of $S_\e(h_i)$ and $C_\phi^*(1/2)$:
    \begin{align*}
        & S_\e(\tilde h_i)(\bx)=S_\e(h_i)(\bx)&\text{for }\bx\in \ov{B_{n-2\e}(\zero)}\\
        &S_\e(\tilde h_i)(\bx)\leq \max(S_\e(h_i)(\bx),C_\phi^*(1/2))\leq S_\e(h_i)+C_\phi^*(1/2)&\text{for }\bx\in \ov{B_{n}(\zero)}\\
        &S_\e(\tilde h_i)=C_\phi^*(1/2)&\text{for }\bx \in \ov{B_n(\zero)}^C
    \end{align*}
   
    Now for each $i$, these bounds imply that 
    \begin{align*}
        &\int S_\e(\tilde h_i^n)d\PP_i
        \leq \int_{\ov{B_{n-2\e}(\zero)}} S_\e(h_i^n)d\PP_i +\int_{\ov{B_n(\zero)}-\ov{B_{n-2\e}(\zero)}} S_\e(h_i^n)+C_\phi^*\left(\frac 12\right)d\PP_i
        +\int_{\ov{B_n(\zero)}^{C}} C_\phi^*\left(\frac 12 \right)d\PP_i \nonumber\\
        &=\int_{\ov{B_n(\zero)}} S_\e(h_i^n)d\PP_i+\int_{\ov{B_{n-2\e}(\zero)}^C}C_\phi^*\left(\frac 12\right) d\PP_i
    \end{align*}

    Then, applying this bound for each $i$,
    \begin{align*}
        &\Theta(\tilde h_0^n,\tilde h_1^n)=\int S_\e(\tilde h_1^n)d\PP_1+\int S_\e(\tilde h_0^n)d\PP_0\\ 
        &\leq \left(\int_{\ov{B_n(\zero)}}S_\e(h_1^n)d\PP_1+\int_{\ov{B_n(\zero)}}S_\e(h_0^n)d\PP_0\right)
        +\left(\int_{\ov{B_{n-2\e}(\zero)}^C} C_\phi^*\left(\frac 12\right) d\PP_1+\int_{\ov{B_{n-2\e}(\zero)}^C}C_\phi^*\left(\frac 12\right) d\PP_0 \right)\\
        &=\Theta^n(h_0^n,h_1^n)+C_\phi^*\left(\frac 12 \right)\left(\PP_0(\ov{B_{n-2\e}(\zero)}^C)+\PP_1(\ov{B_{n-2\e}(\zero)}^C)\right)
        \leq \left(\inf_{(h_0,h_1)\in S_\phi}\Theta^n(h_0,h_1)+\delta\right)+\delta C_\phi^*\left(\frac 12\right)
    \end{align*} The last inequality follows from Equations~\ref{eq:delta_choice} and \ref{eq:h_0^n_h_1^n}.
    Because $\delta$ arbitrary, \eqref{eq:compact_approximation} holds.

\end{proof}

\section{Complimentary Slackness}\label{app:complimentary_slackness}

     

    

\lemmacomplimentaryslacknessrestricted*
 
    Notice that the forward direction of this lemma is actually a consequence of the approximate complimentary slackness result in Lemma~\ref{lemma:approximate_complimentary_slackness}, but we provide a separate self-contained proof below.
 \begin{proof}
 
 First assume that $(\PP_0^*,\PP_1^*)$ maximizes $\bar R_\phi$ over $\Wball \e(\PP_0)\times \Wball \e(\PP_1)$ and $(h_0^*,h_1^*)$ minimizes $\Theta$ over $S_\phi$.
    Because $\PP^*_i\in \Wball\e(\PP_i)$ and $(h_0^*,h_1^*)\in S_\phi$, by Lemma~\ref{lemma:S_e_W_infty_equivalence} 
    \begin{align}
        &\Theta(h_0^*,h_1^*)=\int S_\e(h_1^*)d\PP_1+\int S_\e(h_0^*)d\PP_0\geq \int h_1^*d\PP_1^*+\int h_0^*d\PP_0^*\label{eq:c_slack_first_restricted}\\
        &=\int \eta^* h_1^*+\left(1-\eta^* \right) h_0^* d\PP^*\geq \int C_\phi^*(\eta^*)d\PP^*=\bar R_\phi(\PP_0^*,\PP_1^*)\label{eq:comp_slackness_restricted}    
    \end{align}
    By Lemma~\ref{lemma:duality_theta_all}, both the first expression of \eqref{eq:c_slack_first_restricted}  and the last expression of \eqref{eq:comp_slackness_restricted} are equal. Thus all the inequalities above must be equalities which implies 
    \eqref{eq:complimentary_slackness_restricted_necessary}.
    Next, because \eqref{eq:comp_slackness_restricted} implies that 
    \[\int S_\e(h_1^*)d\PP_1+\int S_\e(h_0^*)d\PP_0= \int h_1^*d\PP_1^*+\int h_0^*d\PP_0^*\]
    and Lemma~\ref{lemma:S_e_W_infty_equivalence} implies that $\int S_\e(h_0^*)d\PP_0\geq \int h_0^*d\PP_0^*$ and $\int S_\e(h_1^*)d\PP_1\geq \int h_1^*d\PP_1^*$ we can conclude \eqref{eq:sup_comp_slack_restricted}. 

    We will now show the opposite implication. Assume that $h_0^*,h_1^*,\PP_0^*,\PP_1^*$ satisfy \eqref{eq:sup_comp_slack_restricted} and \eqref{eq:complimentary_slackness_restricted_necessary}. Then
    \begin{align*}
        &\Theta(h_0^*,h_1^*)=\int S_\e(h_1^*)d\PP_1+\int S_\e(h_0^*)d\PP_0\\
        &=\int h_1^*d\PP_1^*+\int h_0^*d\PP_0^*&\text{(Equation~\ref{eq:sup_comp_slack_restricted})}\\
        &=\int \eta^* h_1^*+(1-\eta^*)h_0^*d\PP^*=\int C_\phi^*(\eta^*)d\PP^*&\text{(Equation~\ref{eq:complimentary_slackness_restricted_necessary})}\\
        &=\dl(\PP_0^*,\PP_1^*)    
    \end{align*}
     However, Lemma~\ref{lemma:duality_theta_all} implies that $\Theta(h_0,h_1)\geq \dl(\PP_0',\PP_1')$ for \emph{any} $h_0,h_1,\PP_0',\PP_1'$. Therefore, $h_0^*,h_1^*$ must be optimal for $\Theta$ and $\PP_0^*,\PP_1^*$ must be optimal for $\dl$.
    
\end{proof}
Notably, a similar strategy shows that if $(h_0^n,h_1^n)\in S_\phi$ is a sequence that satisfies \ref{it:h_0,1_limit} and \ref{it:C_phi_limit} of Lemma~\ref{lemma:approximate_complimentary_slackness}, then $(h_0^n,h_1^n)$ must be a minimizing sequence for $\Theta$.
\section{Technical Lemmas from Section~\ref{sec:primal_existence}}
\subsection{Proof of Lemma~\ref{lemma:exponential_loss}}\label{app:exponential_loss}

 \lemmaexponentialloss*
	\begin{proof}
		First, one can verify that $-\infty$ minimizes $C_\psi(0,\alpha)$ and $\infty$ minimizes $C_\psi(1,\alpha)$, and that $C_\psi^*(0)=C_\psi^*(1)=0$. 
		To find minimizers of $C_\psi(\eta,\alpha)$ for $\eta\in (0,1)$, we solve $\partial_\alpha C_\psi(\eta,\alpha)=-\eta e^{-\alpha}+(1-\eta)e^\alpha =0$, resulting in $\alpha_\psi(\eta)=1/2 \log(\eta/1-\eta)$.
		This formula allows for computation of $C_\psi^*(\eta)$ via $C_\psi^*(\eta)=C_\psi(\eta,\alpha_\psi(\eta))$.
		
		Next, by definition
		\[\eta \psi(\alpha_\psi(\eta))+(1-\eta)(-\psi(\alpha_\psi(\eta)))= C_\psi^*(\eta)\quad \text{ and } s \psi(\alpha_\psi(\eta))+(1-s)(-\psi(\alpha_\psi(\eta)))\geq C_\psi^*(s)\] for all $s\in[0,1]$.
		Therefore, $\psi(\alpha_\psi(\eta))-\psi(-\alpha_\psi(\eta))$ is a supergradient of $C_\psi^*(\eta)$ at $\eta$. 
		
		The function $C_\psi^*$ is differentiable on $(0,1)$, and thus the superdifferential is unique on this set. 
        To show that $\partial C_\psi^*(0)$, $\partial C_\psi^*(1)$ are singletons, it suffices to observe that 
		\[\lim_{\eta\to 0} \frac d {d\eta} C_\psi^*(\eta)=+\infty, \lim_{\eta\to 1} \frac d {d\eta} C_\psi^*(\eta)=-\infty.\]
	\end{proof}

\subsection{Proof of Lemma~\ref{lemma:a_n_b_n}}\label{app:a_n_b_n}

    \lemmaanbn*
	\begin{proof}
		Recall that on the extended real number line, every subsequence has a convergent subsequence. We will show that $\lim_{n\to \infty} a_n=\psi(\alpha_\psi(\eta_0))$ and $\lim_{n\to \infty} b_n=\psi(-\alpha_\psi(\eta_0))$ by proving that every convergent subsequence of $\{a_n\}$ converges to $\psi(\alpha_\psi(\eta_0))$ and every convergent subsequence of $b_n$ converges to $\psi(\alpha_\psi(\eta_0))$.
		
		Let $a_{n_k}$, $b_{n_k}$ be a convergent subsequences of $\{a_n\}$, $\{b_n\}$ respectively. (Again, this convergence is in $\ov \Rset$.) Set $a=\lim_{k\to \infty} a_{n_k}$, $b=\lim_{k\to\infty} b_{n_k}$.
		
		Then \eqref{eq:eta_all_ineq} \eqref{eq:eta_0_limit} imply  that
		\[\eta a+(1-\eta)b\geq C_\psi^*(\eta)\text{ for all }\eta\in [0,1]\]
		\begin{equation}\label{eq:regular_a_b}
			\eta_0 a+(1-\eta_0) b=C_\psi^*(\eta_0)
		\end{equation}
		These equations imply that $a-b \in \partial C_\psi^*(\eta_0)$ and thus 
		\begin{equation}\label{eq:subdiff_a_b}
			a-b=\psi(\alpha_\psi(\eta_0))-\psi(-\alpha_\psi(\eta_0))
		\end{equation}
		while \eqref{eq:regular_a_b} is equivalent to 
		\begin{equation}\label{eq:regular_a_b_simplified}
			\eta_0 a+(1-\eta_0) b=\eta_0\psi(\alpha_\psi(\eta_0))+(1-\eta_0)\psi(-\alpha_\psi(\eta_0))
		\end{equation}
	
		The equations \eqref{eq:subdiff_a_b} and \eqref{eq:regular_a_b_simplified} comprise a system of equations in two variables with a unique solution for $a$ and $b$.
	\end{proof}
	
\subsection{Proof of Lemma~\ref{lemma:sup_liminf_swap}}\label{app:S_e_liminf_limsup_swap}
Lastly, we prove Lemma~\ref{lemma:sup_liminf_swap}.
\lemmasupinfswap*
\begin{proof}
    We start by showing \eqref{eq:sup_liminf_swap}. 
    \begin{align*}
    &\liminf_{n\to \infty} S_\e(h_n)(\bx)=\liminf_{n\to \infty} \sup_{\|\bh\|\leq \e} h_n(\bx+\bh)=\sup_N \inf_{n\geq N} \sup_{\|\bh\|\leq \e}h_n(\bx+\bh)\\
    &\geq  \sup_{\|\bh\|\leq \e}\sup_N \inf_{n\geq N}h_n(\bx+\bh) = \sup_{\|\bh\|\leq \e} \liminf_{n\to \infty} h_n(\bx+\bh)= S_\e( \liminf_{n\to \infty} h_n)(\bx)
\end{align*}

    Equation~\ref{eq:sup_limsup_swap} can then be proved by the same argument:
    \begin{align*}
    &\limsup_{n\to \infty} S_\e(h_n)(\bx)=\limsup_{n\to \infty} \sup_{\|\bh\|\leq \e} h_n(\bx+\bh)=\inf_N \sup_{n\geq N} \sup_{\|\bh\|\leq \e}h_n(\bx+\bh)\\
    &\geq  \sup_{\|\bh\|\leq \e}\inf_N \sup_{n\geq N}h_n(\bx+\bh) = \sup_{\|\bh\|\leq \e} \limsup_{n\to \infty} h_n(\bx+\bh)= S_\e( \limsup_{n\to \infty} h_n)(\bx)
\end{align*}
\end{proof}

\bibliography{bibliography.bib,bib2.bib}
\end{document}